\documentclass[twoside,11pt]{article}

\usepackage{jmlr2e}
\usepackage{chngcntr}
\usepackage{apptools}
\usepackage{amsmath}
\usepackage{amsfonts}
\usepackage{amstext}
\usepackage{float}
\usepackage{color}
\usepackage{enumerate}
\usepackage{float}
\usepackage{subfig}
\usepackage{setspace}
\usepackage{multirow}
\usepackage[ruled,vlined,linesnumbered]{algorithm2e}

\DeclareMathOperator*{\argmax}{arg\,max}

\newcommand{\R}{\mathbb{R}}
\newcommand{\E}{\mathbb{E}}

\usepackage{lastpage}
\jmlrheading{18}{2018}{1-\pageref{LastPage}}{11/17; Revised 
3/18}{4/18}{17-653}
{Qianxiao Li, Long Chen, Cheng Tai, Weinan E}

\ShortHeadings{Maximum Principle Based Algorithms for Deep Learning}{Li, Chen, Tai and E}

\begin{document}
\title{Maximum Principle Based Algorithms for Deep Learning}

\author{\name Qianxiao Li \email liqix@ihpc.a-star.edu.sg \\
	\addr Institute of High Performance Computing\\
	Agency for Science, Technology and Research\\
	1 Fusionopolis Way, Connexis North, Singapore 138632
	\AND
	\name Long Chen \email xidonglc@pku.edu.cn \\
	\addr Peking University\\
	Beijing, China, 100080
	\AND	
	\name Cheng Tai \email chengtai@pku.edu.cn \\
	\addr Beijing Institute of Big Data Research\\
	and Peking University\\
	Beijing, China, 100080
	\AND
	\name Weinan E \email weinan@math.princeton.edu \\
	\addr 
	Princeton University\\
	Princeton, NJ 08544, USA,\\
    Beijing Institute of Big Data Research\\
	and Peking University\\
	Beijing, China, 100080\\
}

\editor{Yoshua Bengio}

\maketitle

\begin{abstract}%
The continuous dynamical system approach to deep learning is explored in order to devise alternative
frameworks for training algorithms. Training is recast as a control problem and this allows us to formulate
necessary optimality conditions in continuous time using the Pontryagin's maximum principle (PMP).
A modification of the method of successive approximations is then used to solve the PMP, giving rise to an alternative training algorithm for deep learning. This approach has the advantage that rigorous error estimates and convergence results can be established. We also show that it may avoid some pitfalls of gradient-based methods, such as slow convergence on flat landscapes near saddle points. 
Furthermore, we demonstrate that it obtains favorable initial convergence rate per-iteration, provided Hamiltonian maximization can be efficiently carried out - a step which is still in need of improvement. Overall, the approach opens up new avenues to attack problems associated with deep learning, such as trapping in slow manifolds and inapplicability of gradient-based methods for discrete trainable variables. 
\end{abstract}


\section{Introduction}
\label{sec:intro}
Supervised learning using deep neural networks has become an increasingly successful tool in modern machine learning applications~\citep{bengio2009learning,schmidhuber2015deep,lecun2015deep,goodfellow2016deep}. Efficient training methods of very deep neural networks, however, remain an active area of research. The most commonly applied training method is stochastic gradient descent~\citep{robbins1951stochastic,bottou2010large} and its variants~\citep{duchi2011adaptive,zeiler2012adadelta,kingma2014adam,johnson2013accelerating}, where incremental updates to the trainable parameters are performed using gradient information computed via back-propagation~\citep{kelley1960gradient,bryson1975applied}. While efficient to implement, the incremental updates to the parameter tend to be slow, especially in the initial stages of the training. Moreover, other than the computation of gradients through back-propagation, the specific structure of deep neural networks is not exploited. These observations point to the question of whether there exists alternative training methods tailored to deep neural networks. 

In a series of papers, we introduce an alternative approach by exploring the optimal control viewpoint of deep learning~\citep{e2017proposal}. Our focus will be on ideas and algorithms derived from the powerful Pontryagin's maximum principle~\citep{boltyanskii1960theory,pontryagin1987mathematical}, which has two major components: the Hamiltonian dynamics and the condition that at each time the optimal parameters maximize the Hamiltonian. The second component suggests that optimization can be performed independently at different layers. One can also derive an explicit error control estimate based on the maximum principle (see Lemma~\ref{lem:error_estimate} below).

In this first paper, we will consider the simplest context in which the deep neural networks are replaced by continuous (or discretized) dynamical systems, and devise numerical algorithms that are based on the optimality  conditions in the Pontryagin's maximum principle. This leads to a new approach for training deep learning models  that have certain advantages, such as fast initial descent and resilience to stalling in flat landscapes. An additional advantage is that one has a good control of the error through explicit estimates. 

The rest of the paper is organized as follows. In Section~\ref{sec:cts_time_formulation}, we present a dynamical systems viewpoint of function approximation and deep learning. We then discuss the necessary optimality conditions, which is the well-known Pontryagin's maximum principle. In Section~\ref{sec:msa} and~\ref{sec:discrete_time}, we discuss numerical methods to solve the necessary conditions and obtain error estimates and convergence guarantees. Using benchmarking examples, we then compare our method with traditional gradient-descent based methods for optimizing deep neural networks in Section~\ref{sec:numerical}. In Section~\ref{sec:discussion}, we discuss and compare our work with existing literature. Conclusion and outlook are given in Section~\ref{sec:conclusion}. 

\section{Function Approximation by Dynamical Systems}
\label{sec:cts_time_formulation}

We start with a description of the (continuous) dynamical systems approach to machine learning (see~\citealt{e2017proposal}). 
The essential task of supervised learning is to approximate some function
\[
F: \mathcal{X} \rightarrow \mathcal{Y}
\]
which maps inputs in $\mathcal{X}\subset\R^d$ (e.g.~images, time-series) to labels in $\mathcal{Y}$ (categories, numerical predictions). Given a collection of $K$ sample input-label pairs $\{x^i,y^i=F(x^i)\}_{i=1}^K$, one aims to approximate $F$ using these data points. In the dynamical systems framework, we consider the inputs $x=(x^1,\dots,x^K)\in\mathbb{R}^{d\times K}$ as the initial condition of a system of ordinary differential equations
\begin{equation}
\dot{X}^i_t = f(t,X^i_t,\theta_t), \quad X^i_0=x^i, \quad 0\leq t\leq T,
\label{eq:dyn_sys_res_nn}
\end{equation}
where $\theta: [0,T]\rightarrow\Theta\subset\mathbb{R}^p$, represents the control (training) parameters and $X_t=(X_t^1,\dots,X_t^K)\in\mathbb{R}^{d\times K}$ for all $t\in[0,T]$. 
The form of $f$ is chosen as part of the machine learning model. For example, in deep learning, $f$ is typically the composition (in either order) of a linear transformation and a component-wise nonlinear function (the activation function). 
For the solution to~\eqref{eq:dyn_sys_res_nn} to exist for any $\theta$, we shall assume hereafter that $f$ and $\nabla_x f$ are continuous in $t,x,\theta$. Note that weaker but more complicated conditions can be considered~\citep{clarke2005maximum}. 
For the $i^\text{th}$ input sample, the prediction of the ``network'' is a deterministic transformation of the terminal state $g(X^i_T)$ for some $g:\mathbb{R}^d\rightarrow \mathcal{Y}$, which we can view collectively as a function of the initial state (input) $x^i$ and the control parameters (weights) $\theta$. The dynamics~\eqref{eq:dyn_sys_res_nn} are decoupled across samples except for the dependence on the control $\theta$.
We shall consider quite a general space of controls
\[
\mathcal{U} := \{\theta:[0,T]\rightarrow\Theta:\theta \text{ is Lebesgue measurable}\}.
\]
The aim is to select $\theta$ from $\mathcal{U}$ so that $g(X^i_T)$ most closely resembles $y^i$ for $i=1,\dots,K$. To this end, we define a loss function $\Phi:\mathcal{Y}\times \mathcal{Y}\rightarrow \R$ which is minimized when its arguments are equal, and we consider minimizing $\sum_i\Phi(g(x^i_T),y^i)$. Since $g$ is fixed, we shall absorb it into the definition of the loss function by defining $\Phi_i(\cdot) := \Phi(\cdot,y_i)$. Then, the supervised learning problem in our framework is
\begin{align}
&\min_{\theta \in \mathcal{U}} \sum_{i=1}^{K}\Phi_i(X^i_T) + \int_{0}^{T} L(\theta_t)dt, \nonumber\\
&\dot{X}^i_t = f(t,X^i_t,\theta_t), \quad X^i_0=x^i, \quad 0\leq t\leq T, \quad i=1,\dots,K,
\label{eq:dyn_sys_opt_prob}
\end{align}
where $L:\Theta\rightarrow\R$ is a running cost, or the regularizer\footnote{We can also make $L$ depend on $X_t$, but for simplicity of presentation and the fact that most current machine learning models do not regularize the states, we shall omit this general case.}. We note here that alternatively, we can formulate the supervised learning problem more generally in terms of optimal control in function spaces, see Appendix~\ref{sec:Function space formulation}.  

Problem~\eqref{eq:dyn_sys_opt_prob} is a special case of a class of general optimal control problem for ordinary differential equations~\citep{bertsekas1995dynamic,athans2013optimal}. The advantage of this formulation is that we can write down and study the optimality conditions of~\eqref{eq:dyn_sys_opt_prob} entirely in continuous time and derive numerical algorithms that can subsequently be discretized. In other words, we {\it optimize, then discretize}, as opposed to the traditional reverse approach in deep learning. 

As was suggested in~\cite{e2017proposal}, deep residual networks~\citep{he2016deep} can be considered as the forward Euler discretization
of the continuous approach described above. In this connection,  the algorithms presented in this paper can also be
formulated in the context of deep residual networks.
For general deep neural networks, although one can also formulate similar algorithms,  it is not clear at this moment
that PMP holds and these algorithms are valid (e.g.~converge to the right solution)  in the general setting.
This issue will be studied in future work. 

The optimization problem~\eqref{eq:dyn_sys_opt_prob} can be solved by first discretizing it into a discrete problem (a feed-forward neural network) and then applying back propagation and gradient descent approaches commonly used in deep learning. However, here we will present an alternative approach. Hereafter, for simplicity of notation we shall set $K=1$ drop the scripts $i$ on all functions, noting that analogous results can be obtained in the general case since the dynamics and loss functions are decoupled across samples. Equivalently, we can think of this as effectively concatenating all $K$ sample inputs into a single input vector of dimension $d\times K$ and redefine our dynamics accordingly. Hence, all results remain valid if we perform full-batch training. The case of mini-batch training is discussed in Section~\ref{sec:remark_minibatch}. 

\subsection{Pontryagin's Maximum Principle}
\label{sec:pmp}
In this section, we introduce a set of necessary conditions for optimal solutions of~\eqref{eq:dyn_sys_opt_prob}, known as the Pontryagin's Maximum Principle (PMP)~\citep{boltyanskii1960theory,pontryagin1987mathematical}. This shall pave  way for an alternative numerical algorithm to train~\eqref{eq:dyn_sys_opt_prob} and its discrete-time counter-part. 

To begin with, we define the {\it Hamiltonian} $H:[0,T]\times \R^d\times \R^d\times \Theta \rightarrow \R$ given by
\[
H(t,x,p,\theta) := p \cdot f(t,x,\theta) - L(\theta).
\]
\begin{theorem}[Pontryagin's Maximum Principle]
\label{thm:pmp}
Let $\theta^*\in \mathcal{U}$ be an essentially bounded optimal control, i.e.~a solution to~\eqref{eq:dyn_sys_opt_prob} with $\text{ess\,sup}_{t\in[0,T]} \Vert \theta_t^*\Vert_\infty<\infty$ ($\text{ess\,sup}$ denotes the essential supremum). Denote by $X^*$ the corresponding optimally controlled state process. Then, there exists an absolutely continuous co-state process $P^*:[0,T]\rightarrow \R^d$ such that the Hamilton's equations
\begin{align}
\dot{X}^*_t &= \nabla_p H(t,X^*_t,P^*_t,\theta^*_t), & X^*_0 &= x, \label{eq:pmp_state_eqn}\\
\dot{P}^*_t &= -\nabla_x H(t,X^*_t,P^*_t,\theta^*_t), & P^*_T &= -\nabla\Phi(X^*_T), \label{eq:pmp_costate_eqn}
\end{align}
are satisfied. Moreover, for each $t\in[0,T]$, we have the Hamiltonian maximization condition
\begin{equation}
H(t,X^*_t, P^*_t, \theta^*_t) \geq H(t,X^*_t, P^*_t, \theta) \text{ for all } \theta\in\Theta.
\label{eq:pmp_hamiltonian_maximization}
\end{equation}
\end{theorem}
	
The proof of the PMP and its variants can be found in any optimal control theory reference, e.g.~\cite{athans2013optimal,bertsekas1995dynamic,liberzon2012calculus}. Some generalizations can be found in~\cite{clarke2005maximum} and references therein. For example, the requirement of the continuity of $f$ with respect to $t$ can be replaced by a much weaker measurability requirement if one assumes more conditions on $\nabla_x f$. 
In the statement of Theorem~\ref{thm:pmp}, we omitted a technicality involving an abnormal multiplier: the terminal condition for $P^*$ should be $P^*_T = -\lambda \nabla\Phi(X^*_T)$ and the Hamiltonian should be defined as
$H(t,x,p,\theta)=p\cdot f(t,x,\theta) - \lambda L(\theta)$ for some $\lambda\geq 0$ (abnormal multiplier) that we can choose. When we are forced to always take $\lambda=0$, the problem is {\it singular} and in a sense ill-posed~\citep{athans2013optimal}. On the contrary, if we can take a positive $\lambda$, we can then rescale the equation for $P^*$ so that we can take $\lambda=1$ without loss of generality. We shall hereafter assume that this is the case. 

A few remarks are in order. First, Equation~\ref{eq:pmp_state_eqn},~\ref{eq:pmp_costate_eqn} and~\ref{eq:pmp_hamiltonian_maximization} allow us to solve for the unknowns $X^*,P^*,\theta^*$ simultaneously as a function of $t$. In this sense, the resulting optimal control $\theta^*$ is {\it open-loop} and is not in a feed-back form $\theta^*_t = \theta^*(X^*_t)$. The latter is of closed-loop type and are typically obtained from dynamic programming and the Hamilton-Jacobi-Bellman formalism~\citep{bellman2013dynamic}. In this sense, the PMP gives a weaker control. However, open-loop solutions are sufficient for neural network applications, where the trained weights and biases are fixed and only depend on the layer number and not the inputs. 

PMP can be regarded as a (highly non-trivial) generalization of the calculus of variations to non-smooth settings (since we only assume $\theta^*$ to be measurable). Perhaps more familiar to the optimization community, the PMP is related to the Karush-Kuhn-Tucker (KKT) conditions for non-linear constrained optimization. Indeed, we can view~\eqref{eq:dyn_sys_opt_prob} as a non-linear program over the function space $\mathcal{U}$ where the constraint is the ODE~\eqref{eq:dyn_sys_res_nn}. In this sense, the co-state process $P^*$ plays the role of a continuous-time analogue of Lagrange multipliers. The key difference between the PMP and the KKT conditions (besides the lack of inequality constraints on the state) is the Hamiltonian maximization condition~\eqref{eq:pmp_hamiltonian_maximization}, which is stronger than a typical first-order condition that assumes smoothness with respect to $\theta$ (e.g.~$\nabla_\theta H=0$). In particular, the PMP says that $H$ is not only stationary, but globally maximized at an optimal control - which is a much stronger statement if $H$ is not concave. Moreover, the PMP makes minimal assumptions on the parameter space $\Theta$; the PMP holds even when $f$ is non-smooth with respect to $\theta$, or worse, when $\Theta$ is a discrete subset of $\R^p$. 

Last, we emphasize that the PMP is only a necessary condition, hence there can be cases where solutions to the PMP is not actually globally optimal for~\eqref{eq:dyn_sys_opt_prob}. Nevertheless, in practice the PMP is often strong enough to give good solution candidates, and when certain convexity assumptions are satisfied the PMP becomes sufficient~\citep{bressan2007introduction}. In the next section, we will discuss numerical methods that can be used to solve the PMP.

\section{Method of Successive Approximations}
\label{sec:msa}
Now, our strategy is to devise numerical algorithms for training~\eqref{eq:dyn_sys_opt_prob} via solving the PMP (Equation~\ref{eq:pmp_state_eqn},~\ref{eq:pmp_costate_eqn} and~\ref{eq:pmp_hamiltonian_maximization}). We derive and analyze algorithms entirely in continuous time, which allows us to characterize errors estimates and convergence in a more transparent fashion. 

There are many methods for the numerical solution of the PMP, including two-point boundary value problem method~\citep{bryson1975applied,roberts1972two}, and collocation methods~\citep{betts1998survey} coupled with general non-linear programming techniques~\citep{bertsekas1999nonlinear,bazaraa2013nonlinear}. See~\citep{rao2009survey} for a more recent review. However, many of these methods concern small-scale problems typically encountered in control applications (e.g.~trajectory optimization of spacecrafts) and do not scale well to modern machine learning problems with a large number of state and control variables. One exception is the method of successive approximations (MSA)~\citep{chernousko1982method}, which is an iterative method based on alternating propagation and optimization steps. We first introduce the simplest form of the MSA. 

\subsection{Basic MSA}
\label{sec:basic_msa}
Observe that~\eqref{eq:pmp_state_eqn} is simply the equation
\[
\dot{X}^*_t = f(t,X^*_t,\theta^*_t),
\]
and is independent of the co-state $P^*$. Therefore, we may proceed in the following manner. First, we make an initial guess of the optimal control $\theta^0\in\mathcal{U}$. For each $k=0,1,2,\dots$, we first solve~\eqref{eq:pmp_state_eqn}
\begin{equation}
\dot{X}^{\theta^k}_t = f(t,X^{\theta^k}_t,\theta^k_t), \quad X^{\theta^k}_0 = x. 
\label{eq:Xtheta}
\end{equation}
for $X^{\theta^k}$, which then allows us to solve~\eqref{eq:pmp_costate_eqn}
\begin{equation}
\dot{P}^{\theta^k}_t = -\nabla_x H(t,X^{\theta^k}_t,P^{\theta^k}_t,\theta^k_t), \quad P^{\theta^k}_T = -\nabla \Phi(X^{\theta^k}_T),
\label{eq:Ptheta}
\end{equation}
to get $P^{\theta^k}$. Finally, we use the maximization condition~\eqref{eq:pmp_hamiltonian_maximization} to set
\[
\theta^{k+1}_t = \argmax_{\theta\in\Theta} H(t,X^{\theta^k}_t,P^{\theta^k}_t,\theta),
\]
for $t\in[0,T]$. The algorithm is summarized in Algorithm~\ref{alg:basic_msa}. 

\begin{algorithm}
	\SetAlgoLined
	Initialize: $\theta^0\in \mathcal{U}$ \;
	\For{$k = 0$ \KwTo \#Iterations}{
		Solve $\dot{X}^{\theta^k}_t = f(t,X^{\theta^k}_t,\theta^k_t), \quad X^{\theta^k}_0 = x$\;
		Solve $\dot{P}^{\theta^k}_t = -\nabla_x H(t,X^{\theta^k}_t,P^{\theta^k}_t,\theta^k_t), \quad P^{\theta^k}_T = -\nabla \Phi(X^{\theta^k}_T)$\;
		Set $\theta^{k+1}_t = \argmax_{\theta\in\Theta} H(t,X^{\theta^k}_t,P^{\theta^k}_t,\theta)$ for each $t\in[0,T]$\;
	}
	\caption{Basic MSA}
	\label{alg:basic_msa}
\end{algorithm} 

As is the case with  the maximum principle,  MSA consists of two major components:  the forward-backward Hamiltonian dynamics and the maximization for the optimal parameters at each time.
An important feature of MSA is that the Hamiltonian maximization step is decoupled for each $t\in[0,T]$. In the language of deep learning, the optimization step is decoupled for different network layers and only the Hamiltonian ODEs (Step 3,4 of Algorithm~\ref{alg:basic_msa}) involve propagation through the layers. 
This allows the parallelization of the maximization step, which is typically the most time-consuming step. 

It has been shown that the basic MSA converges for a restricted class of linear quadratic regulators~\citep{aleksandrov1968accumulation}.
However, in general it tends to diverge, especially if a bad initial $\theta^{0}$ is chosen~\citep{aleksandrov1968accumulation,chernousko1982method}.
Our goal now is to modify the basic MSA to control its divergent behavior. Before we do so, it is important to understand why the MSA diverges, and in particular, the relationship between the maximization step in Algorithm~\ref{alg:basic_msa} and the optimization problem~\eqref{eq:dyn_sys_opt_prob}. 

\subsection{Error Estimate for the Basic MSA}
\label{sec:msa_error_est}
For each $\theta\in\mathcal{U}$, let us denote 
\[
J(\theta) := \Phi(X^{\theta}_T) + \int_{0}^{T} L(\theta_t) dt,
\]
where $X^\theta$ satisfies~\eqref{eq:Xtheta}. Our goal is to minimize $J(\theta)$. We show in the following Lemma the relationship between the values of $J$ and the Hamiltonian maximization step. We start by making the following assumptions.
\begin{enumerate}
	\item[(A1)] $\Phi$ is twice continuously differentiable, with $\Phi$ and $\nabla \Phi$ satisfying a Lipschitz condition, i.e.~there exists $K>0$ such that
	\[
	\vert\Phi(x)-\Phi(x')\vert + \Vert\nabla \Phi(x)-\nabla \Phi(x')\Vert \leq K \Vert x-x' \Vert,
	\]
	for all $x,x'\in \R^d$. 
	\item[(A2)] $f(t,\cdot,\theta)$ is twice continuously differentiable in $x$, with  $f,\nabla_x f$ satisfying a Lipschitz condition in $x$ uniformly in $\theta$ and $t$, i.e.~there exists $K>0$ such that
	\[
	\Vert f(t,x,\theta)-f(t,x',\theta)\Vert + \Vert \nabla_x f(t,x,\theta)-\nabla_x f(t,x',\theta)\Vert_2 \leq K \Vert x-x' \Vert,
	\]
	for all $x,x'\in\R^d$ and $t\in[0,T]$. Note that $\Vert\cdot\Vert_2$ denotes the induced 2-norm. 
\end{enumerate}
With these assumptions, we have the following estimate:
\begin{lemma}
	Suppose (A1)-(A2) holds. 
	Then, there exists a constant $C>0$ such that for any $\theta,\phi\in\mathcal{U}$, 
	\begin{align*}
	J(\phi) \leq& J(\theta) - \int_{0}^{T} \Delta_{\phi,\theta} H(t) dt \\ 
	&+ C \int_{0}^{T} \Vert f(t,X^\theta_t,\phi_t) - f(t,X^\theta_t,\theta_t) \Vert^2 dt \\
	&+ C \int_{0}^{T} \Vert \nabla_x H(t,X^\theta_t, P^\theta_t, \phi_t) - \nabla_x H(t,X^\theta_t, P^\theta_t, \theta_t) \Vert^2 dt,
	\end{align*}
	where $X^\theta$, $P^\theta$ satisfy Equations~\ref{eq:Xtheta},~\ref{eq:Ptheta} respectively and $\Delta H_{\phi,\theta}$ denotes the change in Hamiltonian
	\[
	\Delta H_{\phi,\theta}(t) := H(t,X^\theta_t,P^\theta_t, \phi_t) - H(t,X^\theta_t, P^\theta_t, \theta_t).	
	\]
	\label{lem:error_estimate}
\end{lemma}
\begin{proof}
See Appendix~\ref{sec:proof_lemma_1} for the proof and discussion on relaxing the assumptions. 
\end{proof}

In essence, Lemma~\ref{lem:error_estimate} says that the Hamiltonian maximization step in MSA (step 5 in Algorithm~\ref{alg:basic_msa}) is in some sense the optimal descent direction for $J$. However, the last two terms on the right hand side indicates that this descent can be nullified if substituting $\phi$ for $\theta$ incurs too much error in the Hamiltonian dynamics (step 3,4 in Algorithm~\ref{alg:basic_msa}). In other words, the last two integrals measure the degree of satisfaction of the Hamiltonian dynamics~\eqref{eq:pmp_state_eqn},~\eqref{eq:pmp_costate_eqn}, which can be viewed as a feasibility condition, when one replaces $\theta$ by $\phi$. Hence, we shall hereafter refer to these errors as {\it feasibility errors}. 
The divergence of the basic MSA happens when the feasibility errors blow up. Armed with this understanding, we can then modify the basic MSA to ensure convergence. 

\subsection{Extended PMP and Extended MSA}
\label{sec:emsa}
As discussed previously in Lemma~\ref{lem:error_estimate}, the decrement of $J$ is ensured if we can control the feasibility errors in the Hamiltonian dynamics in steps 3,4 of Algorithm~\ref{alg:basic_msa}. To this end, we employ a similar idea to augmented Lagrangians~\citep{hestenes1969multiplier}. Fix some $\rho>0$ and introduce the augmented Hamiltonian
\begin{align}
\tilde{H}(t,x,p,\theta,v,q) :=& H(t,x,p,\theta) - \frac{1}{2} \rho \Vert v-f(t,x,\theta) \Vert^2 \nonumber\\
& - \frac{1}{2} \rho \Vert q+\nabla_x H(t,x,p,\theta) \Vert^2. 
\label{eq:aug_H}
\end{align}
Then, we have the following set of alternative necessary conditions for optimality:
\begin{proposition}[Extended PMP]
	\label{prop:epmp}
	Suppose that $\theta^*$ is an essentially bounded solution to the optimal control problem~\eqref{eq:dyn_sys_opt_prob}. Then, there exists an absolutely continuous co-state process $P^*$ such that the tuple $(X_t^*, P_t^*, \theta_t^*)$ satisfies the necessary conditions
	\begin{align}
	&\dot{X}_{t}^{*} =  \nabla_{p}\tilde{H}(t,X_{t}^{*},P_{t}^{*},\theta_{t}^{*},\dot{X}_t^*,\dot{P}_t^*),&
	X_{0}^{*}&=x, \label{eq:ext_pmp_state_eqn}\\
	&\dot{P}_{t}^{*} = -\nabla_{x}\tilde{H}(t,X_{t}^{*},P_{t}^{*},\theta_{t}^{*},\dot{X}_t^*,\dot{P}_t^*),&
	P_{T}^{*}&=-\nabla_{x}\Phi(X_{T}^{*}), \label{eq:ext_pmp_costate_eqn}\\
	&\tilde{H}(t,X_{t}^{*},P_{t}^{*},\theta^*_t,\dot{X}_t^*,\dot{P}_t^*) \geq \tilde{H}(t,X_{t}^{*},P_{t}^{*},\theta,\dot{X}_t^*,\dot{P}_t^*),
	& \theta&\in\Theta,t\in [0,T].
	\label{eq:ext_pmp_hamiltonian_maximization}
	\end{align}
\end{proposition}
\begin{proof}
	If $\theta^*$ is optimal, then by the PMP there exists a co-state process $P^*$ such that~\eqref{eq:pmp_state_eqn},~\eqref{eq:pmp_costate_eqn} and~\eqref{eq:pmp_hamiltonian_maximization} are satisfied. Then, for all $t\in[0,T]$ and $\theta\in\Theta$ we have 
	\begin{align*}
	\nabla_x \tilde{H}(t,X_{t}^{*},P_{t}^{*},\theta,\dot{X}_t^*,\dot{P}_t^*) =& \nabla_x H(t,X_{t}^{*},P_{t}^{*},\theta), \\
	\nabla_p \tilde{H}(t,X_{t}^{*},P_{t}^{*},\theta,\dot{X}_t^*,\dot{P}_t^*) =& \nabla_p H(t,X_{t}^{*},P_{t}^{*},\theta),	
	\end{align*}
	which implies that~\eqref{eq:ext_pmp_state_eqn} and~\eqref{eq:ext_pmp_costate_eqn} are satisfied. Lastly, we can write
	\begin{align*}
	&\tilde{H}(t,X_{t}^{*},P_{t}^{*},\theta,\dot{X}_t^*,\dot{P}_t^*) \\
	= &H(t,X_{t}^{*},P_{t}^{*},\theta) 
	- \frac{1}{2}\rho \Vert \dot{X}_t^* - f(t,X_t^*,\theta) \Vert^2	
	- \frac{1}{2}\rho \Vert \dot{P}_t^* + \nabla_x H(t,X_t^*,P_t^*,\theta) \Vert^2.
	\end{align*}
	For each $t$, $\theta^*$ maximizes all three terms on the RHS simultaneously, and hence~\eqref{eq:ext_pmp_hamiltonian_maximization} is also satisfied. 
\end{proof}
Compared with the usual PMP, the extended PMP is a weaker necessary condition. However, the advantage is that maximization of $\tilde{H}$ naturally penalizes errors in the Hamiltonian dynamical equations, and hence we should expect MSA applied to the extended PMP to converge for large enough $\rho$. Note that the Hamiltonian equation steps do not change (since the added terms have no effect on optimal solutions) and the only change is the maximization step. The extended MSA (E-MSA) algorithm is summarized in Algorithm~\ref{alg:extended_msa}.

To establish convergence, define
\[
\mu_k:=\int_{0}^{T}\Delta H_{\theta^{k+1},\theta^k}(t)dt \geq 0.
\]
If $\mu_k=0$, then from the Hamiltonian maximization step~\eqref{eq:ext_pmp_hamiltonian_maximization} we must have 

\begin{align*}
0 = &-\mu_k \leq - \frac{1}{2}\rho \int_{0}^{T} \Vert f(t,X^{\theta^k}_t,\theta^{k+1}_t) - f(t,X^{\theta^k}_t,\theta^k_t) \Vert^2 dt \\
&- \frac{1}{2}\rho \int_{0}^{T} \Vert \nabla_x H(t,X^{\theta^k}_t, P^{\theta^k}_t, \theta^{k+1}_t) - \nabla_x H(t,X^{\theta^k}_t, P^{\theta^k}_t, \theta^k_t) \Vert^2 dt. \leq 0.
\end{align*}
and so 
\[
\max_\theta \tilde{H}(X^{\theta^k}_t,P^{\theta^k}_t, \theta, \dot{X}^{\theta^k}_t, \dot{P}^{\theta^k}_t)=\tilde{H}(X^{\theta^k}_t,P^{\theta^k}_t, \theta_k, \dot{X}^{\theta^k}_t, \dot{P}^{\theta^k}_t),
\]
i.e.~$(X^{\theta^k}, P^{\theta^k}, \theta^k)$ satisfy the extended PMP. In other words, the quantity $\mu_k\geq 0$ measures the distance from a solution of the extended PMP, and if it equals 0, then we have a solution. We now prove the following result that guarantees the convergence of the extended MSA (Algorithm~\ref{alg:extended_msa}).
\begin{theorem}
	Let (A1)-(A2) be satisfied and $\theta^0\in\mathcal{U}$ be any initial measurable control with $J(\theta^0)<+\infty$. Suppose also that
	$\inf_{\theta\in\mathcal{U}} J(\theta)>-\infty$. Then, for $\rho$ large enough, we have under Algorithm~\ref{alg:extended_msa},
	\[
	J(\theta^{k+1}) - J(\theta^k) \leq -D \mu_k.
	\]
	for some constant $D>0$ and
	\[
	\lim_{k\rightarrow0} \mu_k = 0,
	\]
	i.e.~the extended MSA algorithm converges to the set of solutions of the extended PMP.
	\label{thm:convergence_emsa}
\end{theorem}
\begin{proof}
	Using Lemma~\ref{lem:error_estimate} with $\theta\equiv\theta^k,\phi\equiv\theta^{k+1}$, we have
	\begin{align}
	J(\theta^{k+1}) - J(\theta^k) \leq &- \mu_k + C \int_{0}^{T} \Vert f(t,X^{\theta^k}_t,\theta^{k+1}_t) - f(t,X^{\theta^k}_t,\theta^k_t) \Vert^2 dt \nonumber\\
	&+ C \int_{0}^{T} \Vert \nabla_x H(t,X^{\theta^k}_t, P^{\theta^k}_t, \theta^{k+1}_t) - \nabla_x H(t,X^{\theta^k}_t, P^{\theta^k}_t, \theta^k_t) \Vert^2 dt. \nonumber
	\end{align}
	From the Hamiltonian maximization step in Algorithm~\ref{alg:extended_msa}, we know that
	\begin{align}
	H(t,X^{\theta^k}_t, P^{\theta^k}_t, \theta^{k}_t) \leq &H(t,X^{\theta^k}_t, P^{\theta^k}_t, \theta^{k+1}_t)  \nonumber\\
	&- \frac{1}{2}\rho \Vert f(t,X^{\theta^k}_t,\theta^{k+1}_t) - f(t,X^{\theta^k}_t,\theta^k_t) \Vert^2 \nonumber\\
	&- \frac{1}{2}\rho \Vert \nabla_x H(t,X^{\theta^k}_t, P^{\theta^k}_t, \theta^{k+1}_t) - \nabla_x H(t,X^{\theta^k}_t, P^{\theta^k}_t, \theta^k_t) \Vert^2. \nonumber
	\end{align}
	Hence, we have
	\begin{align*}
	J(\theta^{k+1}) - J(\theta^k) \leq -(1-\frac{2C}{\rho}) \mu_k.
	\end{align*}
	Pick $\rho>2C$, then we indeed have $J(\theta^{k+1}) - J(\theta^k) \leq -D \mu_k$ with $D=(1-\frac{2C}{\rho})>0$. Moreover, we can rearrange and sum the above expression to get 
	\[
	\sum_{k=0}^{M} \mu_k \leq D^{-1}(J(\theta^0) - J(\theta^{M+1})) \leq D^{-1}(J(\theta^0) - \inf_{\theta\in\mathcal{U}}J(\theta)),
	\]
	and hence $\sum_{k=0}^{\infty} \mu_k<+\infty$, which implies $\mu_k \rightarrow 0$ and the extended MSA converges to a solution of the extended PMP.
\end{proof}

\begin{algorithm}
	\SetAlgoLined
	Initialize: $\theta^0\in \mathcal{U}$. Hyper-parameter: $\rho$ \;
	\For{$k = 0$ \KwTo \#Iterations}{
		Solve $\dot{X}^{\theta^k}_t = f(t,X^{\theta^k}_t,\theta^k_t), \quad X^{\theta^k}_0 = x$\;
		Solve $\dot{P}^{\theta^k}_t = -\nabla_x H(t,X^{\theta^k}_t,P^{\theta^k}_t,\theta^k_t), \quad P^{\theta^k}_T = -\nabla \Phi(X^{\theta^k}_T)$\;
		Set $\theta^{k+1}_t = \argmax_{\theta\in\Theta} \tilde{H}(t,X^{\theta^k}_t,P^{\theta^k}_t,\theta,\dot{X}^{\theta^k}_t,\dot{P}^{\theta^k}_t)$ for each $t\in[0,T]$\;
	}
	\caption{Extended MSA}
	\label{alg:extended_msa}
\end{algorithm} 

\section{Discrete-Time Formulation}
\label{sec:discrete_time}

In the previous section, we discussed the PMP and MSA in the continuous-time setting, where we showed that an appropriately extended version (E-MSA) converges to a solution of an extended PMP. Here, we shall discuss the discretized versions of PMP, MSA and E-MSA, as well as their connections to deep residual networks and back-propagation. 

\subsection{Discrete-Time PMP and Discrete-Time MSA}
Applying Euler-discretization to Equation~\ref{eq:dyn_sys_res_nn}, we get
\[
x_{n+1} = x_n + \delta f_n({x_n,\vartheta_n}), \quad x_0 = x,
\]
for $n=0,\dots,N-1$, with $\delta = T/N$ (step-size), $x_n:=X_{n\delta}$, $\vartheta_n:=\theta_{n\delta}$ and $f_n(\cdot):=f(n\delta,\cdot)$. Then, the discrete-time analogue of the control problem~\eqref{eq:dyn_sys_opt_prob} is 
\begin{align}
&\min_{\{\vartheta_0,\dots,\vartheta_{N-1}\}\in\Theta^{N}} \Phi(x_N) + \delta\sum_{n=0}^{N-1} L(\vartheta_n) , \nonumber\\
&x_{n+1} = x_n + \delta f_n(x_n,\vartheta_n), \quad x_0 = x, \quad 0\leq n\leq N-1. \label{eq:discrete_opt_prob}
\end{align}
Observe that barring the constant $\delta$, this is exactly the supervised learning problem for deep residual networks\footnote{If we pick ReLU activations~\citep{hahnloser2000digital}, then $\delta$ can be absorbed into $\vartheta$}. Therefore, when suitably discretized, one expects that the E-MSA provides a means to train residual neural networks via the solution of the extended PMP.

We now write down formally the discretized form of the PMP. Let us use the shorthand $g_n(x_n,\vartheta_n):=x_n + \delta f_n(x_n, \vartheta_n)$. Define the scaled discrete Hamiltonian
\[
H_n(x,p,\vartheta) = p\cdot g_n(x,\vartheta) - \delta L(\vartheta).
\]
Then, a discrete-time PMP is the following set of conditions:
\begin{align*}
&x^*_{n+1} = g_n(x^*_n,\vartheta^*_n), & &x^*_0=x,\\
&p^*_{n} = \nabla_x H_n(x^*_n,p^*_{n+1},\vartheta_n), & &p^*_N = -\nabla_x \Phi(x^*_N), \\
&H_n(x^*_n,p^*_{n+1},\vartheta^*_{n}) \geq H_n(x^*_n,p^*_{n+1},\vartheta), & &\vartheta\in\Theta,\quad n=0,\dots,N-1.
\end{align*}
The issue of whether the PMP holds for discrete time dynamical systems is a delicate one and there are known counterexamples~\citep{butkovsky1963necessary,jackson1965discrete,nahorski1984discrete}. Nevertheless, they must hold approximately for small time step size and this is the situation we will consider in the current paper. We expect Lemma~\ref{lem:error_estimate}, which implies monotonicity of the E-MSA algorithm, to hold in the discrete-time case under appropriate conditions. We leave a rigorous analysis of these statements to future work. For numerical experiments presented in the next section, we shall almost always work with residual networks that can be regarded as discretizations of continuous networks so that the PMP holds approximately at least~\citep{halkin1966maximum}.  

For completeness, we summarize the discrete-time version of E-MSA in Algorithm~\ref{alg:discrete_emsa}. Note that for residual networks ($g_n=x_n+\delta f_n$), this is equivalent to a forward Euler discretization on the state equation and a backward Euler discretization on the co-state equation in Algorithm~\ref{alg:extended_msa}. As before, the Hamiltonian maximization step is decoupled across layers and can be carried out in parallel. 

\begin{algorithm}
	\SetAlgoLined
	Initialize: Initialize: $\vartheta^0_n \in \Theta_n$, $n=0,\dots,N-1$. Hyper-parameter: $\rho$ \;
	\For{$k = 0$ \KwTo \#Iterations}{
		Set $x^{\theta^k}_0=x$ \;
		\For{$n = 0$ \KwTo $N-1$}{
			$x^{\vartheta^k}_{n+1} = g_n(x^{\vartheta^k}_n,\vartheta^k_n)$ \;
		}
		Set $p^{\vartheta^k}_{N}=-\nabla \Phi(x^{\vartheta^k}_{N})$ \;
		\For{$n = N-1$ \KwTo $0$}{
			$p^{\vartheta^k}_{n} = \nabla_x H_n(x^{\vartheta^k}_n, p^{\vartheta^k}_{n+1}, \vartheta^k_n)$ \;
		}
		\For{$n = 0$ \KwTo $N-1$}{
			Set
			$\vartheta^{k+1}_n = \argmax_{\vartheta\in\Theta_n}
			{H}_n(x^{\vartheta^k}_n,p^{\vartheta^k}_{n+1},\vartheta) 
			- \tfrac{1}{2} \rho \Vert x^{\vartheta^k}_{n+1} - g_n(x^{\vartheta^k}_n,\vartheta) \Vert_2^2
			- \tfrac{1}{2} \rho \Vert p^{\vartheta^k}_{n} - \nabla_x H_n(x^{\vartheta^k}_n,p^{\vartheta^k}_{n+1},\vartheta) \Vert_2^2
			$\;
		}
	}
	\caption{Discrete-time E-MSA}\label{alg:discrete_emsa}
\end{algorithm} 

\subsection{Relationship to Gradient Descent with Back-propagation}
\label{sec:msa_and_bp}
We note an interesting relationship of the MSA with classical gradient descent with back-propagation~\citep{kelley1960gradient,bryson1975applied,lecun1998gradient}. We have shown in Lemma~\ref{lem:error_estimate} that the divergence of MSA can be attributed to the large errors in the Hamiltonian dynamics terms caused by the maximization step, which involve drastic changes in parameter values. Assuming each $\Theta_n$ is a continuum and $g_n$, $L$ are differentiable in $\vartheta_n$, a simple fix is to make the maximization step ``soft'': we replace step 12 in Algorithm~\ref{alg:discrete_emsa} with a gradient ascent step:
\begin{equation}
\vartheta^{k+1}_n = \vartheta^k_n + \eta \nabla_\vartheta H_n(x^{\vartheta^k}_n,p^{\vartheta^k}_{n+1},\vartheta^k_n),
\label{eq:H_grad_ascent}
\end{equation}
for some small learning rate $\eta$. We now show that in the discrete-time setting, this is equivalent to the classical gradient descent with back-propagation. 
\begin{proposition}
	\label{prop:pmp_vs_bp}
	The basic MSA in discrete-time (Algorithm~\ref{alg:discrete_emsa} with $\rho=0$) with step 12 replaced by~\eqref{eq:H_grad_ascent} is equivalent to gradient descent with back-propagation.
\end{proposition}
\begin{proof}
	Recall that the Hamiltonian is
	\[
	H_n := p_{n+1} \cdot g_n(x_n,\vartheta_n) - \delta L(\vartheta_n), 
	\]
	and the total loss function is $J(\vartheta) = \Phi(x_N) + \delta \sum_{n=0}^{N-1} L(\vartheta_n)$. It is easy to see that $p_n = -\nabla_{x_n} \Phi(x_N)$ by working backwards from $n=N$ and the fact that $\nabla_{x_n} x_{n+1} = \nabla_x g_n(x_n,\vartheta_n)$. Then, 
	\begin{align*}
	\nabla_{\vartheta_n} J(\vartheta) =& \nabla_{x_{n+1}} \Phi(x_N) \cdot \nabla_{\vartheta_n} x_{n+1} + \delta \nabla_{\vartheta_n} L(\vartheta_n)\\
	=& - p_{n+1} \cdot \nabla_{\vartheta_n} g_n(x_n,\vartheta_n) + \delta \nabla_{\vartheta_n} L(\vartheta_n)\\
	=& - \nabla_{\vartheta_n} H_n.
	\end{align*}
	Hence,~\eqref{eq:H_grad_ascent} is simply the gradient descent step
	\[
	\vartheta^{k+1}_n = \vartheta^k_n - \eta \nabla_{\vartheta_n} J(\vartheta^k).
	\]
\end{proof}
As the proposition shows, gradient descent with back-propagation can be seen as a modification of the basic MSA by replacing the Hamiltonian maximization step with a gradient ascent step. However, we note that the PMP (and MSA convergence) holds, at least in continuous-time, even when differentiability with respect to $\vartheta$ is not satisfied, and hence is more general than the classical back-propagation. In fact, the PMP formalism shows that the back-propagation of information through a deep network is handled by the co-state equation and there is no requirement or relationship to the gradients with respect to the trainable parameters. In other words, optimization is performed at each layer separately (with or without gradient information), and propagation is independent of optimization. 

\subsection{A Remark on Mini-batch Algorithms}
\label{sec:remark_minibatch}
So far, our discussion has focused on full-batch algorithms, where the input $x$ represents the full set of training inputs.
As modern supervised learning tasks typically involve a large number of training samples, usually the optimization problem has to be solved in mini-batches, where at each iteration we sub-sample $m$ input-label pairs and optimize the parameters $\theta$ (or $\vartheta$ in discrete time) based on losses evaluated on these pairs. 
In the context of continuous-time PMP, we can write the batch version of the three necessary conditions as
\begin{align*}
\dot{X}^{i,*}_t &= \nabla_p H(t,X^{i,*}_t,P^{i,*}_t,\theta^*_t), & X^{i,*}_0 &= x^i,\\
\dot{P}^{i,*}_t &= -\nabla_x H(t,X^{i,*}_t,P^{i,*}_t,\theta^*_t), & P^{i,*}_T &= -\nabla\Phi^i(X^{i,*}_T),\\
\theta^*_t &= \argmax_{\theta\in\Theta} \sum_{i=1}^{M}H(t,X^{i,*}_t, P^{i,*}_t, \theta), & t&\in[0,T],
\end{align*}
for samples $i=1,\dots,M$. We omit for brevity the equivalent expressions for discrete-time. In particular, notice that the propagation steps are decoupled across samples, and hence can be carried out independently. The only difference is the maximization step, where in a mini-batch setting we would evaluate instead
\[
\argmax_{\theta\in\Theta} \sum_{i=1}^{m}H(t,X^{i,*}_t, P^{i,*}_t, \theta).
\]
If $m$ is large enough and the samples are independently and identically drawn, then uniform law of large numbers~\citep{jennrich1969asymptotic} holds under fairly general conditions and ensures that the mini-batch mean of Hamiltonians converges uniformly in $\theta$ to the full-batch sum. Hence, maximization performed on the mini-batch sum should be close to the actual maximization on the full Hamiltonian. Rigorous error estimates for the mini-batch version of our algorithm is out of the scope of the current work, and we use instead numerical results in Section~\ref{sec:numerical} to demonstrate that the algorithm can also be carried out in a mini-batch fashion.  

\section{Numerical Experiments}
\label{sec:numerical}
In this section, we investigate the performance of E-MSA compared with the usual gradient-based approaches, namely stochastic gradient descent and its variants: Adagrad~\citep{duchi2011adaptive} and Adam~\citep{kingma2014adam}. To illustrate key properties of E-MSA, we shall begin by investigating some synthetic examples. First, we consider a simple one-dimensional function approximation problem where we want to approximate $F(x)=\sin(x)$ for $x\in[-\pi,\pi]$ using a continuous time dynamical system. Let $T=5$ and consider
\[
\dot{X}_t = f(X_t, \theta_t) = \tanh(W_t X_t + b_t),
\]
where $\theta_t=(W_t,b_t)\in \R^{5\times 5}\times \R^5$, i.e.~a continuous analogue of a fully connected feed forward neural networks with 5 nodes per layer. To match dimensions, we shall concatenate the input $x$ to form a five dimensional vector of identical components, which is now the initial condition to the dynamical system on $\R^d$. The output of the network is $\sum_{i=1}^{5} X_T^{i}$. and we define the loss function due to one sample to be $\Phi(X_T)=(\sum_{i=1}^{5} X_T^i - \sin(x))^2$. For multiple samples, we average the loss function over all samples in the usual way. We apply E-MSA with discretization size $\delta=0.25$ (giving 20 layers) and compute the Hamiltonian maximization step using 10 iterations of limited memory BFGS method (L-BFGS)~\citep{liu1989limited}. In Figure~\ref{fig:sine_example}(a), we compare the results with gradient descent based optimization approaches, where we observe that E-MSA has favorable convergence rate per-iteration. More interestingly, it is well-known that gradient descent may suffer slow convergence at flat regions or near saddle-points, where the gradients become very small and optimization may stall for a long time. This often occurs as a result of poor initialization of weights and biases~\citep{sutskever2013importance}. Here, we simulate this scenario by initializing all weights and biases ($W_t,b_t$) to be $0$ and observe the optimization process. We see from Figure~\ref{fig:sine_example}(b) that gradient descent based methods are more easily stalled at flat regions. 
We calculated numerically the eigenvalues of the Hessian at this region, which confirms that this is indeed very close to a saddle point.
On the other hand, the Hamiltonian maximization in E-MSA can quickly escape the locally flat regions. One possible reason is that second-order information employed by L-BFGS can off-set the small gradients and provide larger updates.  
\begin{figure}[H]
	\begin{center}
		\subfloat[]{\includegraphics[width=12cm]{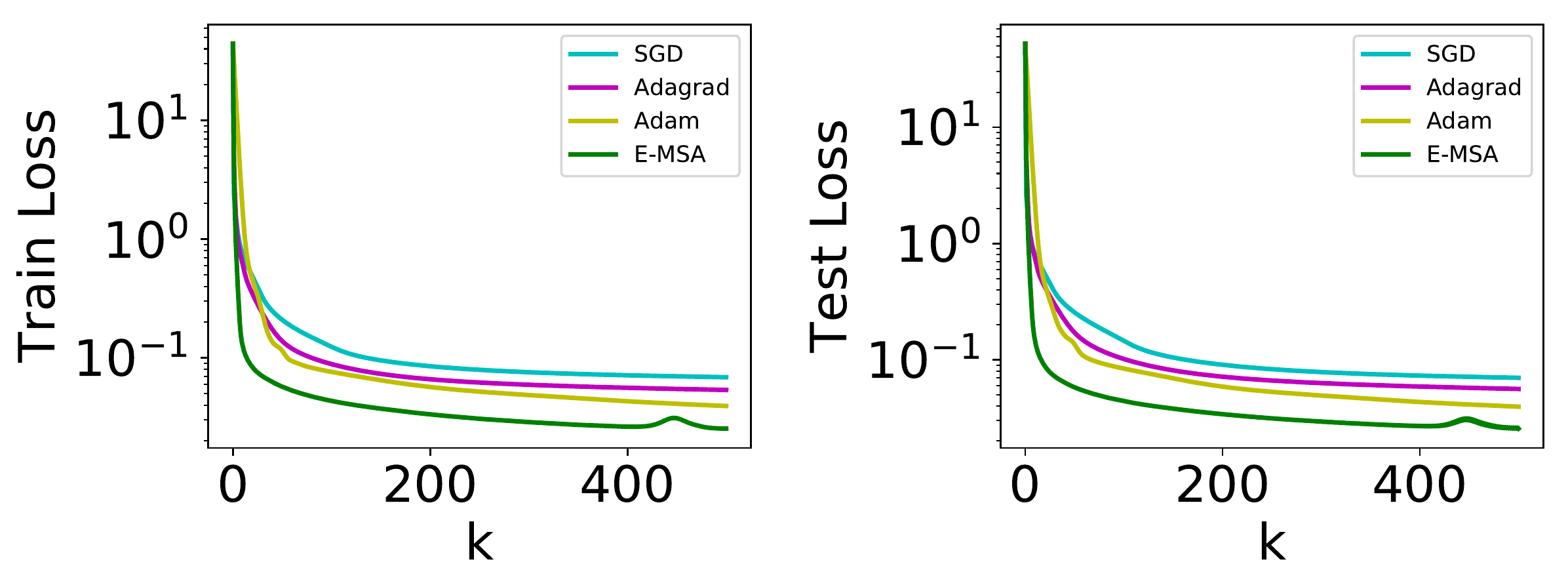}}
		
		\subfloat[]{\includegraphics[width=12cm]{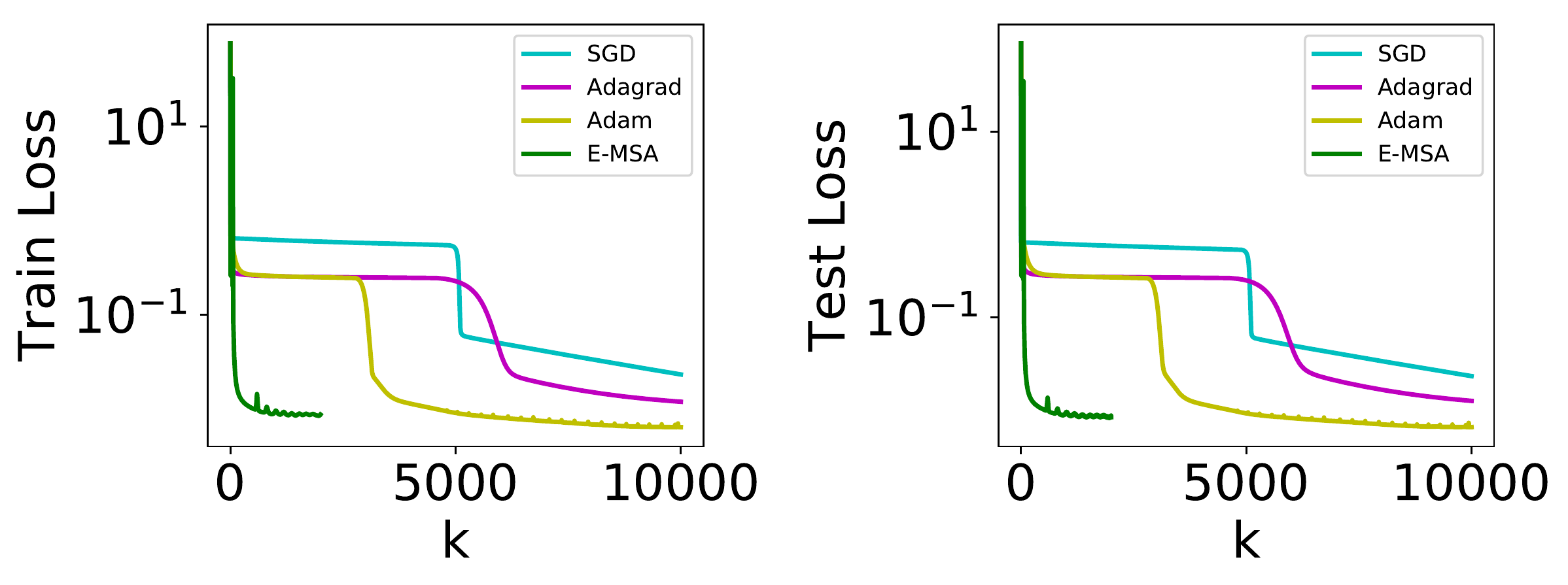}}
	\end{center}
	\caption{Comparison of E-MSA with gradient-based methods for approximating the sine function with a continuous, 5-dimensional dynamical system. A training and test set of 1000 samples each are used.
	(a) Loss function vs iterations for a good initialization, where weights are initialized with truncated random normal variables with standard deviation 0.1 and biases are initialized as constants equal to 0.1. We see that E-MSA has good convergence rate per iteration. 
	(b) We use a poor initialization by setting all weights and biases to 0. We observe that gradient descent based methods tend to become stuck whereas E-MSA are better at escaping these slow manifolds, provided that $\rho$ is well chosen (=1.0 in this case).}
	\label{fig:sine_example}
\end{figure}

Next, we consider a familiar supervised learning test problem: the MNIST data set~\citep{lecun1998mnist} for handwritten digit recognition, with 55000 training samples and 10000 test samples. We employ a continuous dynamical system that resembles a (residual) convolution neural network~\citep{lecun1995convolutional} when discretized. More concretely, at each $t$ we consider the map $f(t,x,\theta) = \tanh(W\star x + b)$ where $W$ is a $3\times 3$ convolution filter with $32$ input and output channels. To match dimensions, we introduce two projection layers at the input (consisting of convolution, point-wise non-linearities followed by $2\times2$ max-pooling). We also use a fully-connected classification layer as the final layer, with softmax cross-entropy loss. Note that the input projection layers and fully-connected output layers are not of residual form, but we can nevertheless apply Algorithm~\ref{alg:discrete_emsa} with the appropriate $g$. We use a total of 10 layers (2 projections, 1 fully-connected and 7 residual layers with $\delta=0.5$, i.e $T=3.5$). The model is trained with mini-batch sizes of 100 using E-MSA and gradient-descent based methods, namely SGD, Adagrad, and Adam. For E-MSA, we approximately solve the Hamiltonian maximization step using either 10 iterations of L-BFGS. Note that since we have decoupled the layers through the PMP, the L-BFGS step used to maximize $H$ is tractable since it involves much fewer parameters than directly minimizing $J$. Figure~\ref{fig:mnist_cnn} compares the performance of E-MSA with the other gradient-descent based methods, where we observe that E-MSA has good performance per-iteration, especially at early stages of training. However, we also show in Figure~\ref{fig:mnist_cnn_wc} that the wall-clock performance of our methods are not currently competitive, because the Hamiltonian maximization step is time consuming and the performance gains per iteration is outweighed by the running time. Note that wall-clock times are compared on the CPU for fairness since we did not use a GPU implementation of L-BFGS. As a further test, we train the same model on a different data set, the fashion MNIST data set~\citep{xiao2017fashion}, where we again observe similar phenomena (see Figure~\ref{fig:fashion_cnn}). Experiments on more complex data sets such as ImageNet~\citep{imagenet_cvpr09} with larger residual networks is a direction of future work. In particular, this may require further improvements to the Hamiltonian maximization step current handled by direct minimization with L-BFGS, which can be significantly slower (on a wall-clock basis) for larger networks and data sets. 
\begin{figure}[hbt]
	\begin{center}
		\subfloat[]{\includegraphics[width=12cm]{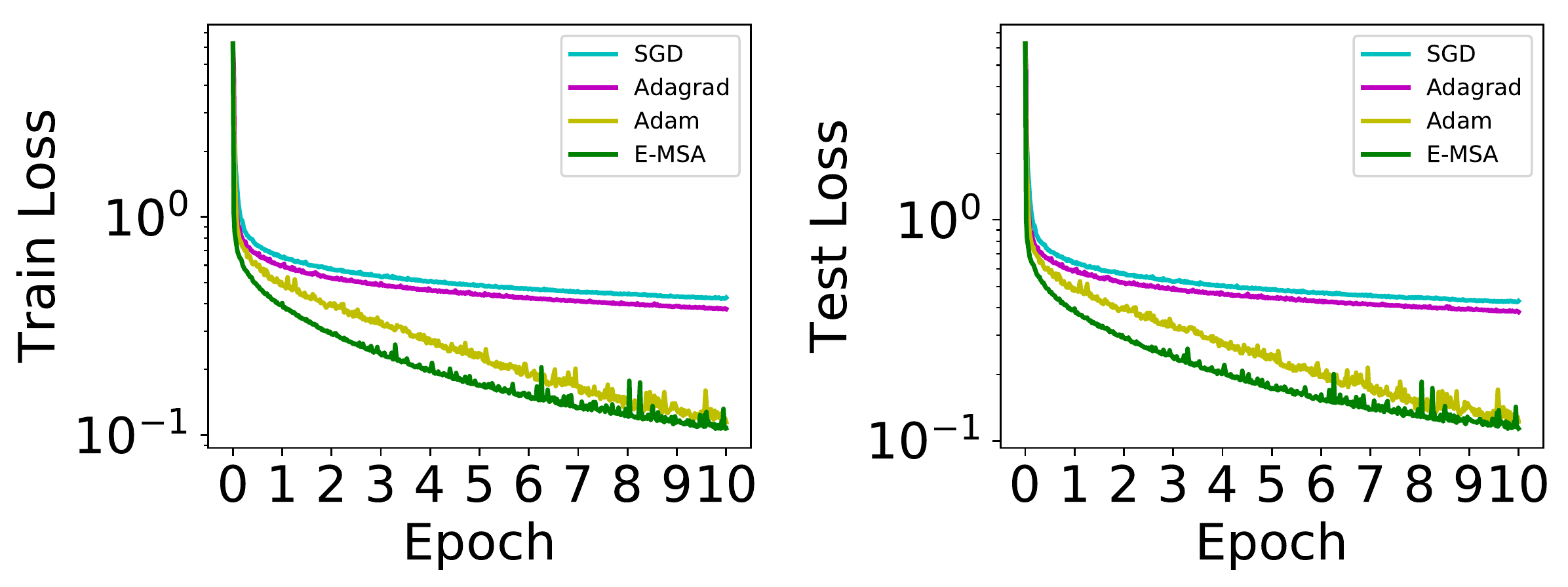}}
		
		\subfloat[]{\includegraphics[width=12cm]{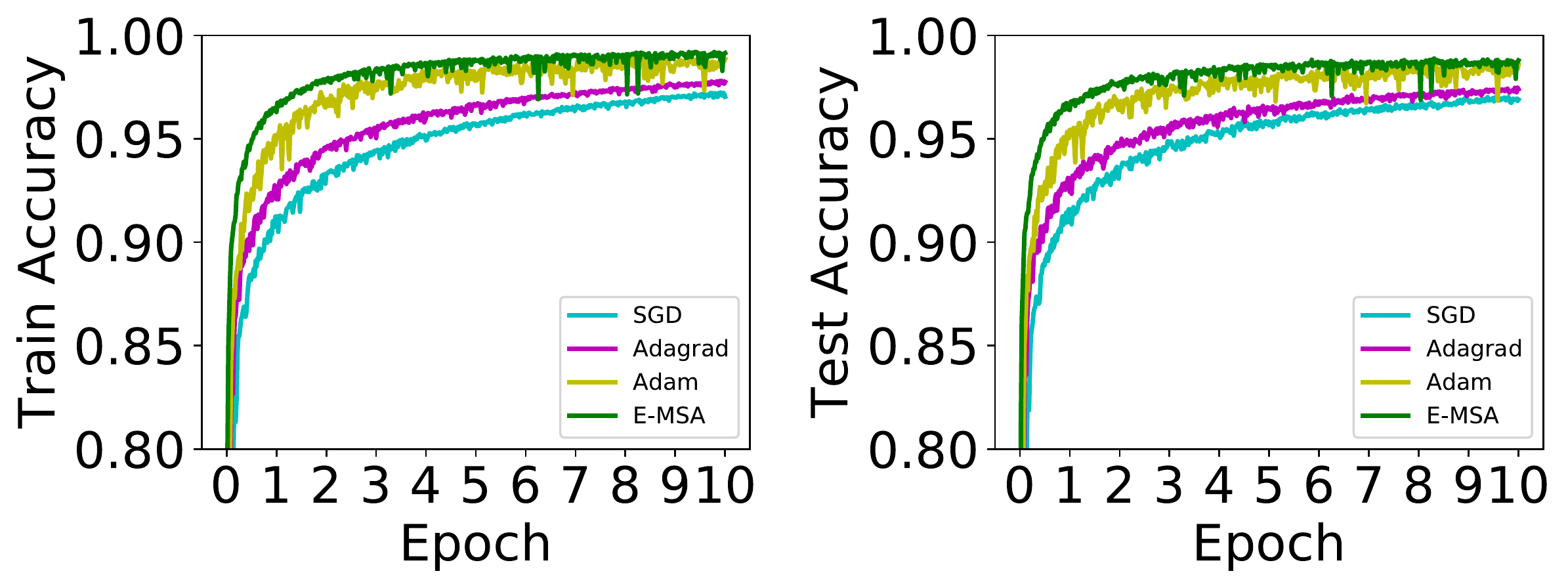}}
	\end{center}
	\caption{Comparison of E-MSA with gradient-based methods for the residual CNN on the MNIST data set. Mini-batch size of 100 is used so that each epoch of training consists of 550 iterations. (a) Train and test Loss vs epoch. (b) Train and test accuracy vs epoch. For each case, we tuned the associated hyper-parameters on a coarse grid for optimal performance. We observe that per-iteration, E-MSA performs favorably, at least at early times. This shows that if the augmented Hamiltonian can be efficiently maximized, we may obtain good performance. }
	\label{fig:mnist_cnn}
\end{figure}
\begin{figure}[hbt]
	\begin{center}
		\includegraphics[width=12cm]{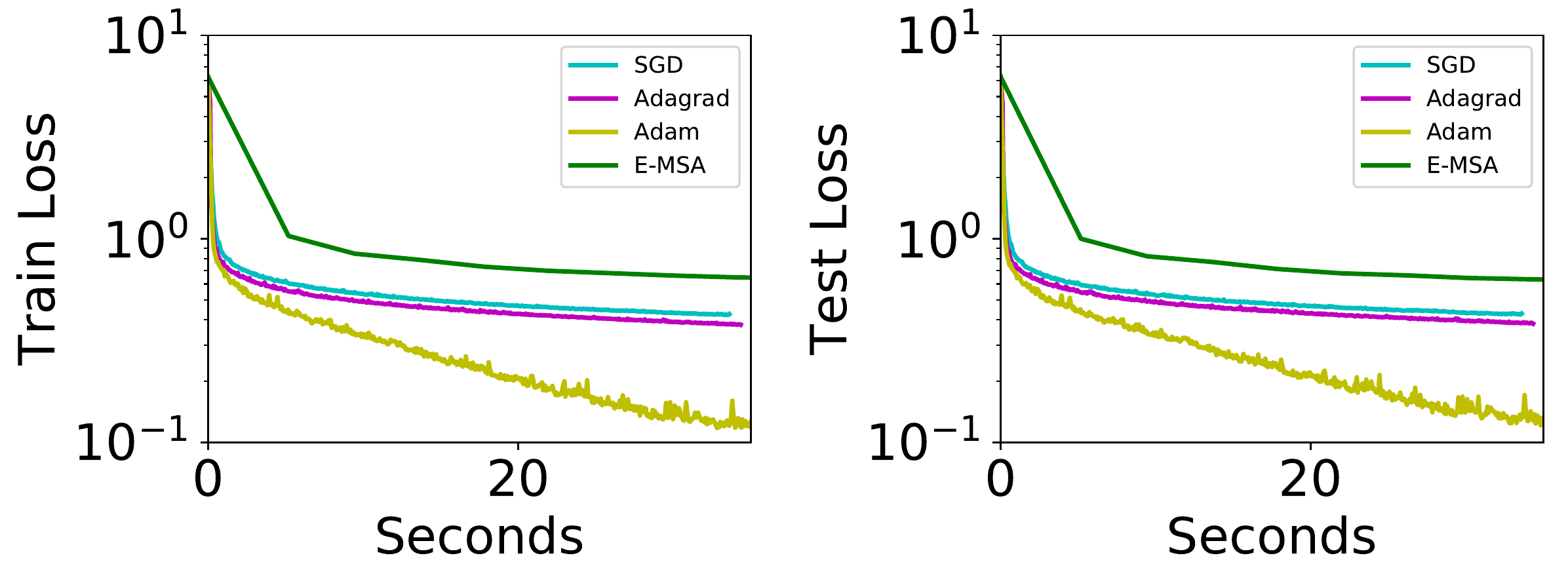}
	\end{center}
	\caption{Comparison of E-MSA with gradient-based methods for the residual CNN on the MNIST data set on a wall-clock basis. We observe that currently, the gains per iteration is outweighed by the additional computational costs. Note that we did not use a GPU implementation for the L-BFGS algorithm used to maximize the augmented Hamiltonian, hence the wall-clock time for E-MSA is expected to be improved. Nevertheless, we expect that more efficient Hamiltonian maximization algorithms must be developed for E-MSA to out-perform gradient-based methods in terms of wall-clock efficiency. }
	\label{fig:mnist_cnn_wc}
\end{figure}
\begin{figure}[hbt]
	\begin{center}
		\subfloat[]{\includegraphics[width=12cm]{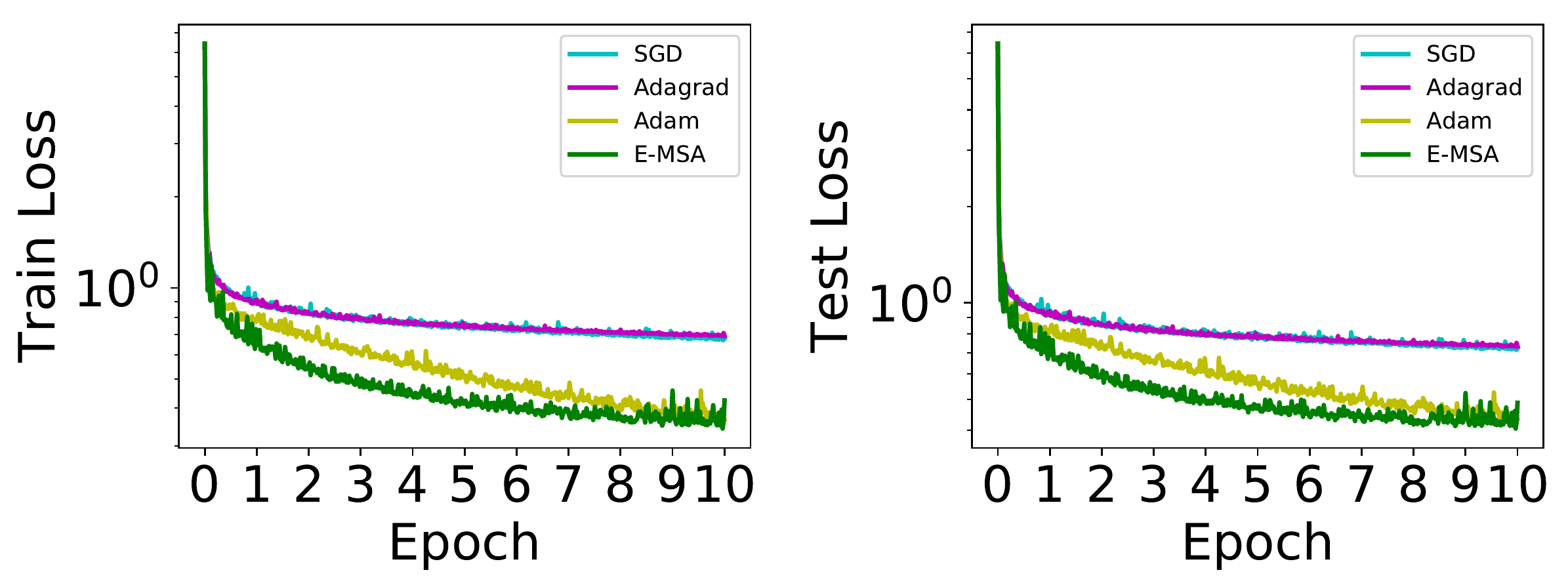}}
		
		\subfloat[]{\includegraphics[width=12cm]{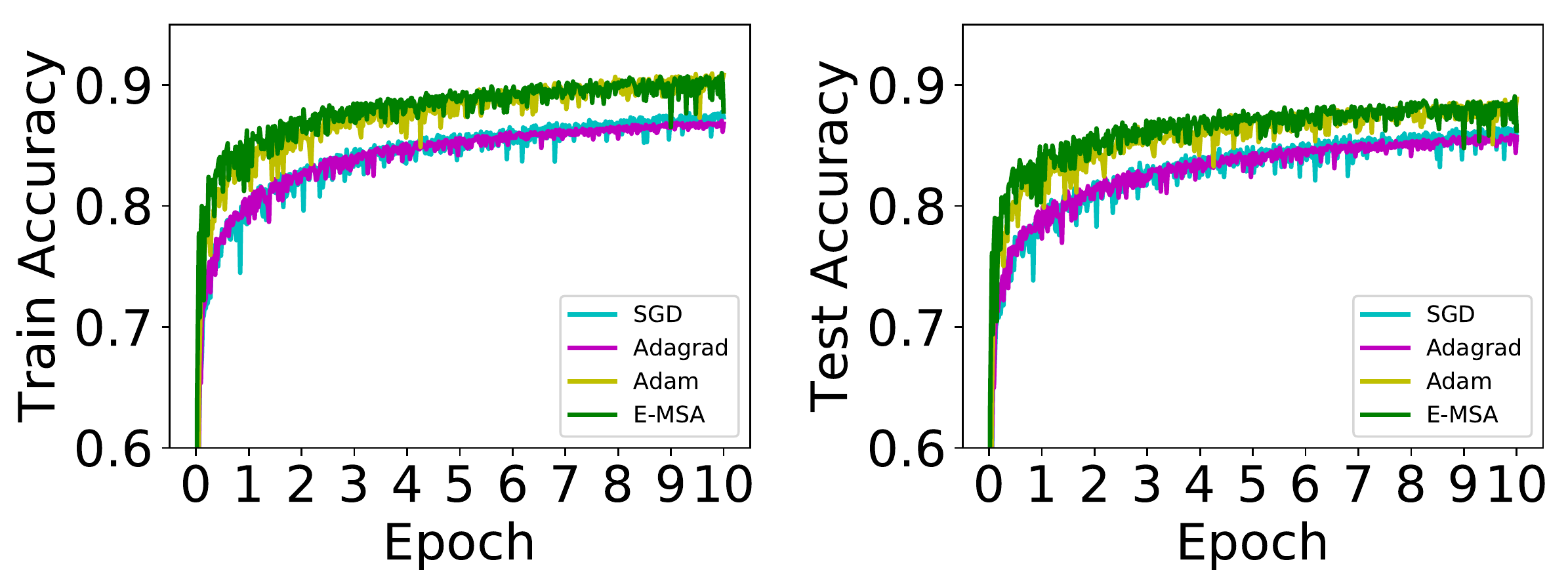}}
	\end{center}
	\caption{Comparison of E-MSA with gradient-based methods for the residual CNN on the fashion MNIST data set. We use the same network structure and mini-batch sizes as in Figure~\ref{fig:mnist_cnn}. The hyper-parameters have to be slightly re-tuned. (a) Train and test Loss vs epoch. (b) Train and test accuracy vs epoch. Again, we observe E-MSA performs favorably per-iteration. }
	\label{fig:fashion_cnn}
\end{figure}

\section{Discussion and Related Work}
\label{sec:discussion} 
We commence this section by highlighting the distinguishing features of E-MSA from traditional gradient-descent based training methods. First, the formulations of PMP and E-MSA do not involve gradient information with respect to the trainable parameters. In fact, Theorem~\ref{thm:pmp} and Algorithm~\ref{alg:extended_msa} remain valid even when the trainable parameters can only take values in a discrete set. Second, due to a more drastic argmax step taken at each iteration, E-MSA tends to have better convergence rates at the early steps of training, as observed in our numerical experiments (Section~\ref{sec:numerical}). Third, in the PMP formalism, the Hamiltonian equations for the state and co-state are the ``forward and backward propagations'', whereas given the state and co-state values, the optimization step is decoupled across layers. This allows one to potentially parallelize the often time-consuming optimization step. Moreover, from Lemma~\ref{lem:error_estimate}, we show that as long as the Hamiltonian is sufficiently increased in a layer without causing too much loss in the Hamiltonian dynamics feasibility conditions, we can ensure decrement of the loss function. This is the reason why we can use a small number of iterations of L-BFGS at each step. Moreover, this suggests that the argmax updates need not happen synchronously, i.e.~the optimization in each layer can be a separate thread or process that computes the argmax and updates that layer's parameters independent of other layers. The propagation may also potentially be allowed to happen asynchronously as long as updates are sufficiently frequent. We leave a rigorous analysis of an asynchronous version of the current approach to future work. In summary, the main strength of the PMP (over e.g.~solving the KKT conditions using gradient methods) is that PMP says that at the optimum, the Hamiltonian is not only stationary (KKT), but globally maximized. This hints that heuristic global optimization methods can be applied to $H$ to obtain algorithms that are very different in behavior compared with gradient-descent based approaches. Again, Lemma~\ref{lem:error_estimate} ensures that such heuristic global maximization need only be approximate. 

As it currently stands, our experiments in Section~\ref{sec:numerical} demonstrate that the Hamiltonian maximization step in E-MSA gives very different behavior compared with gradient-descent based methods. When the Hamiltonian is sufficiently maximized, we indeed obtain favorable performance compared with gradient descent based methods. Furthermore, we saw in Figure~\ref{fig:sine_example} that Hamiltonian maximization may avoid pitfalls such as a very flat landscape. Overall, the key to whether E-MSA (and other methods based on solving the PMP) will eventually constitute a replacement for gradient-descent based algorithm lies in the question of whether efficient Hamiltonian maximization can be performed at reasonable computational costs. Although this is still a non-convex optimization problem, it is much simpler than the original training problem because: (1) Optimization in the layers are decoupled and hence parameter space is greatly reduced; (2) The Hamiltonian is formally similar across different layers, loss functions and models, so specialized algorithms may be designed; (3) The Hamiltonian does not need to be maximized exactly, thus fast heuristic methods~\citep{lee2008modern} or learning~\citep{andrychowicz2016learning,jaderberg2016decoupled,czarnecki2017understanding} can potentially be used to perform this. All these are worthy of future exploration in order to make E-MSA truly competitive. 

Next, we put our work in perspective by discussing related work in the optimal control, optimization and deep learning literature. First, the work on numerical algorithms for the solution of optimal control problem is abundant (see e.g.~\citealt{rao2009survey} for a survey). Many of the state-of-the-art techniques in the control theory literature assume a moderately small problem size, so that conventional non-linear programming techniques~\citep{bertsekas1999nonlinear,bazaraa2013nonlinear} as well as shooting~\citep{roberts1972two} and collocation methods~\citep{betts1998survey} produce efficient algorithms. This is usually not the case for large-scale machine learning problems, where often, the only scalable approach is to rely on iterative updates to the parameters. This is the reason for our focus on the MSA algorithms~\citep{chernousko1982method}, as they are straight-forward to implement and typically have linear scaling in computational complexity with respect to the input and parameter sizes. The basic MSA is discussed in~\cite{krylov1962msa}, and a number of improved variants are discussed in~\cite{chernousko1982method} and references	therein. For example, a popular improvement is based on needle-perturbations, where controls are varied on small intervals at each iteration. While convergent, the main issue with the needle-perturbation approach is the requirement of a sufficiently fine mesh (i.e.~many layers in the discretized network), which impacts computational speed. A possible solution is the use of adaptive meshes, which is a future direction we plan to investigate. Our variant of the MSA presented in this work differs from classical approaches~\citep{chernousko1982method} mainly in the sense that we solve a weaker sufficient condition (extended PMP, Proposition~\ref{prop:epmp}), which then allows us to control errors in the Hamiltonian dynamical equations at every iteration without going into finer mesh-sizes. The regularization terms proportional to $\rho$ is similar to the heuristic modifications suggested in~\cite{lyubushin1982modifications} by regularizing the distance between $\theta^{k}$ and $\theta^{k+1}$, but we do not have to assume convexity of $\Theta$ or that $f$ is Lipschitz in $\theta$. 

In the optimization literature, our work shares some similarity with the recently proposed ADMM methods~\citep{taylor2016training} for training deep neural networks, where the authors also considered necessary conditions with Lagrange multipliers that can decouple optimization across layers. The main difference in our work is that the PMP gives a stronger necessary condition (Hamiltonian maximization) that also applies to general parameter spaces (e.g., discrete, or bounded with non-linear constraints). Our modification of the basic MSA in terms of the augmented Hamiltonian is inspired by the method of augmented Lagrangians often applied in constrained optimization~\citep{hestenes1969multiplier}. The idea of viewing an initially discrete system as the discretization of a continuous-time system has been explored in~\cite{li2017stochastic} in the form of stochastic optimization. Our current work is also in this flavor, but for neural network models.

In deep learning, there are a few works that share our perspective of deep neural networks as a discretization of a dynamical system. We note that the connection between the PMP and back-propagation has been pointed out qualitatively in~\cite{le1988theoretical} and in the development of back-propagation~\citep{bryson1975applied,baydin2015automatic}, although to the best of our knowledge, this work is the first attempt to translate numerical algorithms for the PMP into training algorithms for deep learning that goes beyond gradient descent. The treatment of machine learning as function approximation via a dynamical system has been presented in~\cite{e2017proposal}. The recent work of~\cite{haber2017stable,chang2017reversible} also propose the dynamical systems viewpoint, and the authors used continuous-time tools to address stability issues. In contrast, our work focuses on the optimization aspects centered around the PMP. We also mention other recent approaches to decouple optimization in deep neural networks, such as synthetic gradients~\citep{jaderberg2016decoupled,czarnecki2017understanding} and proximal back-propagation~\citep{frerix2017proximal}. 

\section{Conclusion and Outlook}
\label{sec:conclusion}
In this paper, we discuss the viewpoint that deep residual neural networks can be viewed as discretization of a continuous-time dynamical system, and hence supervised deep learning can be regarded as solving an optimal control problem in continuous time. We explore a concrete consequence of this connection, by modifying the classical method of successive approximations for solving optimal control problems (in particular the PMP) into a method for solving a weaker sufficient condition (extended PMP). We prove the convergence of the resulting algorithm (E-MSA) and test it on various benchmark problems, where we observe that the E-MSA algorithm performs favorably on a per-iteration basis, especially at early stages of training, compared with gradient-based approaches such as SGD, Adagrad and Adam. 

There are many avenues of future research. On the algorithmic side, it is necessary to further improve the computational efficiency of the E-MSA, in particular the Hamiltonian maximization step. Moreover, adaptive selection of $\rho$ depending on iteration number and/or layer can be explored, e.g.~by designing adaptive tuning schemes using control theoretic tools~\citep{li2017stochastic}. Also, it is desirable to formulate and analyze the PMP and E-MSA from a discrete-time perspective in order to broaden the method's application. From a modeling perspective, viewing deep neural networks as continuous-time dynamical systems is useful in the sense that it allows one to think of neural network architectures as dynamical objects. Indeed, at each training iteration of the E-MSA, we do not have to use the same discretization scheme to compute the Hamiltonian dynamical equations. Also, as the PMP and E-MSA assume little structure on the parameter space $\Theta$, it will also be interesting to apply the E-MSA to train neural networks that have discrete weights (e.g.~those that can only take on binary values). Such networks have the advantage of fast inference speed and small memory requirement. However, training such networks is a challenge and most existing techniques rely on approximating or thresholding the derivatives~\citep{courbariaux2015binaryconnect,courbariaux2016binarized}. With the PMP and MSA, we may be able to directly train discrete networks in a principled way. 

\acks{The work of W. E is supported in part by Major Program of NNSFC under grant 91130005, ONR grant N00014-13-1-0338, DOE grants DE-SC0008626 and DE-SC0009248 Q. Li is supported by the Agency for Science, Technology and Research, Singapore. }

\appendix

\section{Function Space Formulation}
\label{sec:Function space formulation}
In this section, we give an alternative, non-rigorous formulation of the supervised learning problem as an optimal control problem on function spaces. This provides an alternative formulation of (continuous-time) deep learning that does not make reference to a specific set of input-outputs, but rather their conditional distributions. The idea is to consider the control of a continuity equation that describes the evolution of probability densities. Hereafter, we proceed formally by assuming all differentiability and integrability conditions are satisfied. 

We would like to approximate, using a dynamical systems approach, some target joint probability density $\rho(x,y)$, where $x\in\mathcal{X}\subset\R^d$ is a sample input and $y\in\mathcal{Y}$ is the corresponding label. In the case where the labels are deterministically determined by the samples, i.e.~there exists $F:\mathcal{X}\rightarrow\mathcal{Y}$ such that $y=F(x)$, we would have $\rho(x,y)=\overline{\rho}(x)\delta(F(x)-y)$. Here, $\overline{\rho}(x)$ is the marginal density of $\rho(x,y)$. In general, we can write $\rho(x,y)=\rho(y|x)\overline{\rho}(x)$. 

As before, the idea is to consider passing the inputs through a dynamical system
\begin{equation}
	\dot{X}_t = f(t,X_t,\theta_t), \qquad X_0=x.
	\label{eq:push_forward}
\end{equation}
We begin with a guess of a conditional density $\rho_0(y|x)$ of $y$ given $x$. In the deterministic case, we may set $\rho_0(y|x) = \delta({y-F_0(x)})$ for some $F_0:\mathcal{X}\rightarrow\mathcal{Y}$ (this is like the last layer of the neural network, be it a regressor or a classifier). Note that $F_0$ is potentially very different from $F$, so that $\rho_0(\cdot|x)$ is far from our target $\rho(\cdot|x)$. 

To improve this approximation, we drive the initial condition by the controllable dynamical system~\eqref{eq:push_forward}. That is, we define the approximation at time $t$ of $\rho(y|x)$ to be $\rho_t(y|x):=\langle \rho_0(y|\cdot), u_t\rangle$, with $u_t$ denoting the probability density of $X_t$ at time $t$ (push-forward distribution of $X_t$ according to~\eqref{eq:push_forward}).  
It is well-known that $u_t$ follows the continuity equation, or Liouville equation~\citep{gibbs2014elementary}; or forward Kolmogorov equation in stochastic processes, but with zero noise~\citep{risken1996fokker}, 
\begin{equation}
	\frac{d}{d t}	u_t = -\text{div}(f(t,\cdot,\theta_t)u_t), \qquad
	u_0 = \delta_{x},
	\label{eq:dyn_sys_pde}
\end{equation}
where $\text{div}u = \sum_i \partial u/\partial x_i$ is the divergence operator and $\delta_x(x')=\delta(x-x')$ is a point-mass at $x$. We shall assume that $u_t\in\mathcal{H}\subset L^2(\R^d)$ for some function space $\mathcal{H}$, for all $t\in(0,T]$. 

The goal now is to adjust $\theta\in\mathcal{U}$ so that $\rho_t(\cdot|x)$ is close to $\rho(\cdot|x)$. 
To this end, we define a differentiable loss function $\Phi(\rho_1,\rho_2)$ that measures distances between two conditional densities $\rho_1,\rho_2$ (e.g., $L^2$ loss, K-L divergence). Then, the learning problem can be formulated as the following optimal control problem:
\begin{align}
&\min_{\theta \in \mathcal{U}} \E_{x\sim \overline{\rho}}
\left[\Phi(\rho_t(\cdot|x),\rho(\cdot|x)) + \int_{0}^{T} L(\theta_t)dt\right], \nonumber \\
&\frac{d}{d t}	u_t = -\text{div}(f(t,\cdot,\theta_t)u_t), \qquad
u_0 = \delta_{x}. 
\label{eq:dyn_sys_pde_opt_prob}
\end{align}
As before, $L$ is a regularizer on the trainable parameters. Now,~\eqref{eq:dyn_sys_pde_opt_prob} is an optimal control problem on the function space $\mathcal{H}$. 

We now write down formally a set of necessary conditions for optimality, in the form of the Pontryagin's maximum principle, for the present function-space control problem~\eqref{eq:dyn_sys_pde_opt_prob}. Define the Hamiltonian functional $H:[0,T]\times \mathcal{H} \times \mathcal{H} \times \Theta \rightarrow \R$
\begin{align*}
	H(t,u,v,\theta) &:= - \langle v, \text{div}(f(t,\cdot,\theta) u) \rangle - L(\theta) \\
	&= -\int_{\R^d} v(x) \sum_{i=1}^d \frac{\partial}{\partial x_i} (f(t,x,\theta) u(x)) dx - L(\theta).
\end{align*}
Then, the Pontryagin's maximum principle for this system is expected to take the form: let $\theta^*\in\mathcal{U}$ be an optimal control, then there exists a co-state process $v_t\in \mathcal{H}$ such that
\begin{align*}
	&\frac{d}{dt} u^*_t = D_v H(t,u^*_t,v^*_t,\theta^*_t), & &u^*_0=\delta_x, \\
	&\frac{d}{dt} v^*_t = -D_u H(t,u^*_t,v^*_t,\theta^*_t), & &v^*_T=-D_u\Phi(\langle\rho_0, u^*_T\rangle,\rho(\cdot|x)) \\
	&\E_{x\sim\overline{\rho}} H(t,u^*_t,v^*_t,\theta^*_t) \geq 
	\E_{x\sim\overline{\rho}} H(t,u^*_t,v^*_t,\theta), & &\theta\in\Theta, t\in[0,T],
\end{align*}
where $D$ denotes the usual Fr\'{e}chet derivative. Note that by definition, we have $D_v H = - \text{div}(f u)$ and $D_u H = f\cdot \nabla_x v$.
Observe that the co-state $v^*$ satisfies the (time-reversed) adjoint Liouville's equation with a specified terminal condition. The PMP for similar functional optimal control problems has been studied in, among others, ~\cite{pogodaev2016optimal,roy2017numerical}, albeit without the expectation over initial density. 

In summary, the advantage of this formulation is that we make no explicit reference to the training data or target functions and formulate the entire problem as a control problem on probability densities. Of course, in practice, to implement an MSA-like algorithm, the terminal condition of the co-state will depend on the target joint density, which we can only access through the sampled data. A rigorous analysis of this function space control formulation and its consequences will be explored in future work. 

\section{Proof of Lemma~\ref{lem:error_estimate}}
\label{sec:proof_lemma_1}
First, observe that assumptions (A1)-(A2) in the main text implies that the second derivatives of $f$ and $\Phi$ are bounded by $K$. Provided that $P^\theta_t$ is bounded, they also imply that the second derivatives of $H$ with respect to $x$ and $p$ are bounded when evaluated on $X^\theta_t,P^\theta_t,\theta_t$. We first establish the boundedness of $P^\theta_t$.
\begin{lemma}
	Assume that (A1)-(A2) hold. Then, there exists a constant $K'>0$ such that for any $\theta$, 
	\[
	\Vert P^\theta_t \Vert \leq K',
	\]
	for all $t\in[0,T]$. 
	\label{lem:H_Deriv}
\end{lemma}
\begin{proof}
	Using~\eqref{eq:Ptheta} and setting $\tau:=T-t$, $\tilde{P}^\theta_\tau := P^\theta_{T-\tau}$ we get
	\[
	\dot{\tilde{P}}^\theta_\tau = \tilde{P}^\theta_\tau \cdot \nabla_x f(t,X^\theta_{T-\tau},T-\tau), \qquad \tilde{P}^\theta_0 = -\nabla \Phi(X^\theta_T).
	\]
	Using (A1)-(A2), we have $\Vert P^\theta_T\Vert=\Vert \nabla_x \Phi(X^\theta_T) \Vert \leq K$ and $\Vert\nabla_x f(t,X^\theta_t,\theta_t)\Vert_2 \leq K$. Hence, 
	\[
	\Vert \dot{\tilde{P}}^\theta_\tau \Vert \leq K \Vert \tilde{P}^\theta_\tau \Vert,
	\]
	and by Gronwall's inequality, 
	\[
	\Vert \tilde{P}^\theta_\tau \Vert \leq K e^{KT} =: K'.
	\]
	This proves the claim since it holds for any $\tau$. 
\end{proof}

We now prove Lemma~\ref{lem:error_estimate}. The approach here is similar to that employed in~\cite{rozonoer1959maximum}. 
\begin{proof}[Proof of Lemma~\ref{lem:error_estimate}]
	From~\eqref{eq:Xtheta} and the definition of the Hamiltonian, we have for any $\theta\in\mathcal{U}$, 
	\[
	I(X^\theta,P^\theta,\theta):=\int_{0}^{T} P^\theta_t \cdot \dot{X}^\theta_t 
	- H(t,X^\theta_t,P^\theta_t,\theta_t)
	- L(\theta_t) dt
	\equiv 0.
	\]
	Denote $\delta X_t = X^\phi_t - X^\theta_t$ and $\delta P_t = P^\phi_t - P^\theta_t$, then we have
	\begin{align}
	0 \equiv& I(X^\phi,P^\phi,\phi) - I(X^\theta, P^\theta, \theta) \nonumber\\
	=& \int_{0}^{T} P^\theta_t \cdot \delta \dot{X}_t + \delta P_t \cdot \dot{X}^\theta_t + \delta P_t \cdot \delta \dot{X}_t  dt \nonumber\\
	&- \int_{0}^{T} H(t,X^\phi_t,P^\phi_t,\phi_t) - H(t,X^\theta_t, P^\theta_t, \theta_t) dt \nonumber \\
	&- \int_{0}^{T} L(\phi_t) - L(\theta_t) dt.
	\label{eq:lem_1_cp1}
	\end{align}
	Now, by integration by parts
	\begin{align}
	\int_{0}^{T} P^\theta_t \cdot \delta \dot{X}_t dt = &P^\theta_t \cdot \delta X_t\Big\vert_0^T 
	- \int_{0}^T \dot{P}^\theta_t \cdot \delta X_t dt, \label{eq:ibp_1}\\
	\int_{0}^{T} \delta P_t \cdot \delta \dot{X}_t dt = &\delta P_t  \cdot \delta X_t\Big\vert_0^T 
	- \int_{0}^T \delta \dot{P}_t  \cdot \delta X_t dt. \label{eq:ibp_2}
	\end{align}
	Using~\eqref{eq:Xtheta},~\eqref{eq:Ptheta} and~\eqref{eq:ibp_1}, we have
	\begin{align}
	&\int_{0}^{T} P^\theta_t  \cdot \delta \dot{X}_t + \delta P_t  \cdot \dot{X}^\theta_t dt \nonumber\\
	= &P^\theta_t  \cdot \delta X_t\Big\vert_0^T
	+ \int_0^T \left(f(t,X^\theta_t;\theta_t) \cdot \delta P + \nabla_x H(t,X^\theta_t, P^\theta_t, \theta_t) \cdot \delta X\right)dt \nonumber\\
	= &P^\theta_t  \cdot \delta X_t\Big\vert_0^T
	+ \int_0^T \left(\nabla_z H(t,Z^\theta_t, \theta_t)  \cdot  \delta Z\right)dt. \label{eq:lem_1_cp2}
	\end{align}
	where in the last line we defined $Z^\theta:=(X^\theta, P^\theta)$. Similarly, from~\eqref{eq:ibp_2} we get
	\begin{align}
	\int_{0}^{T} \delta P_t  \cdot \delta \dot{X}_t dt = & \frac{1}{2}\int_{0}^{T} \delta P_t  \cdot \delta \dot{X}_t dt + \frac{1}{2}\int_{0}^{T} \delta P_t  \cdot \delta \dot{X}_t dt \nonumber\\
	= &\frac{1}{2}\delta P_t  \cdot \delta X_t \Big\vert_0^T \nonumber\\
	&+ \frac{1}{2}\int_0^T \left([\nabla_z H(t,Z^\phi_t, \phi_t) - \nabla_z H(t,Z^\theta_t, \theta_t)] \cdot \delta Z_t\right)dt \nonumber\\
	= &\frac{1}{2}\delta P_t  \cdot \delta X_t \Big\vert_0^T \nonumber\\
	&+ \frac{1}{2}\int_0^T [\nabla_z H(t,Z^\theta_t, \phi_t) - \nabla_z H(t,Z^\theta_t, \theta_t)] \cdot \delta Z_tdt \nonumber\\
	&+ \frac{1}{2}\int_0^T \delta Z_t \cdot  \nabla^2_z H(t,Z^\theta_t + r_1(t) \delta Z_t, \phi_t) \cdot \delta Z_t dt. \label{eq:lem_1_cp3}
	\end{align}	
	where we have used Taylor's theorem in the last step with $r_1(t)\in[0,1]$. We now rewrite the boundary terms. Since $\delta X_0 = 0$, we have
	\begin{align}
	&(P^\theta_t+\frac{1}{2}\delta P_t)  \cdot \delta X_t \Big\vert_0^T = (P^\theta_T+\frac{1}{2}\delta P_T)  \cdot \delta X_T \nonumber\\
	= &-\nabla\Phi(X_T^\theta) \cdot \delta X_T - \frac{1}{2} (\nabla \Phi(X_T^\phi)- \nabla \Phi(X_T^\theta)) \cdot \delta X_T \nonumber\\
	= &-\nabla\Phi(X_T^\theta) \cdot \delta X_T - \frac{1}{2} \delta X_T  \cdot \nabla^2 \Phi(X_T^\theta + r_2\delta X_T) \cdot \delta X_T \nonumber\\
	= & -(\Phi(X^{\phi}_T)-\Phi(X^{\theta}_T)) - \frac{1}{2} \delta X_T  \cdot  (\nabla^2 \Phi(X_T^\theta + r_2\delta X_T)+\nabla^2 \Phi(X_T^\theta + r_3\delta X_T)) \cdot \delta X_T,
	\label{eq:lem_1_cp4}
	\end{align}
	for some $r_2, r_3\in[0,1]$. 
	Lastly, for each $t\in[0,T]$ we have
	\begin{align}
	H(t,Z^\phi_t,\phi_t) - H(t,Z^\theta_t, \theta_t) = &H(t,Z^\theta_t,\phi_t) - H(t,Z^\theta_t, \theta_t)\nonumber\\
	&+\nabla_z H(t,Z^\theta_t,\phi_t) \cdot \delta Z_t \nonumber\\
	&+ \frac{1}{2}\delta Z_t \cdot \nabla^2_z H(t,Z^\theta_t + r_4(t)\delta Z_t, \phi_t)  \cdot \delta Z_t,\label{eq:lem_1_cp5}
	\end{align}
	where $r_4(t)\in[0,1]$. 
	
	Substituting~\eqref{eq:lem_1_cp2},~\eqref{eq:lem_1_cp3},~\eqref{eq:lem_1_cp4},~\eqref{eq:lem_1_cp5} into~\eqref{eq:lem_1_cp1}, we obtain
	\begin{align}
	&\left[\Phi(X^{\phi}_T) + \int_{0}^{T}L(\phi_t)\right] 
	- \left[\Phi(X^{\theta}_T) + \int_{0}^{T}L(\theta_t)\right] \nonumber\\
	=& \frac{1}{2} \delta X_T  \cdot  (\nabla^2 \Phi(X_T^\theta + r_2\delta X_T)+\nabla^2 \Phi(X_T^\theta + r_3\delta X_T)) \cdot \delta X_T \nonumber\\
	& -\int_{0}^{T} \Delta H_{\phi,\theta}(t)dt \nonumber\\
	& +\frac{1}{2}\int_{0}^{T} (\nabla_z H(t,Z^\theta_t,\phi_t)-\nabla_z H(t,Z^\theta_t,\theta_t)) \cdot \delta Z_t dt \nonumber\\
	& +\frac{1}{2}\int_{0}^{T} \left(\delta Z_t \cdot  [\nabla^2_z H(t,Z^\theta_t + r_1(t)\delta Z_t, \phi_t) - \nabla^2_z H(t,Z^\theta_t + r_4(t)\delta Z_t, \phi_t)] \cdot \delta Z_t\right) dt.
	\label{eq:lem_1_cp6}
	\end{align}
	The left hand side is simply $J(\phi)-J(\theta)$, and so it remains to estimate the right hand side terms. First, let us estimate $\delta X$ and $\delta P$. By definition,
	\[
	\delta \dot{X}_t = f(t,X^\phi_t,\phi_t) - f(t,X^\theta_t,\theta_t).
	\]
	Integrating, we get
	\[
	\delta X_t = \int_{0}^{t} f(t,X^\phi_s,\phi_s) - f(t,X^\theta_s,\theta_s) ds,
	\]
	and so
	\begin{align}
	\Vert\delta X_t \Vert \leq &\int_{0}^{t} \Vert f(t,X^\phi_s,\phi_s) - f(t,X^\theta_s,\theta_s)\Vert ds \nonumber\\
	\leq & \int_{0}^{t} \Vert f(t,X^\phi_s,\phi_s) - f(t,X^\theta_s,\phi_s)\Vert ds \nonumber\\
	& +\int_{0}^{t} \Vert f(t,X^\theta_s,\phi_s) - f(t,X^\theta_s,\theta_s)\Vert ds \nonumber\\
	\leq & \int_{0}^{T} \Vert f(t,X^\theta_s,\phi_s) - f(t,X^\theta_s,\theta_s)\Vert ds \nonumber\\
	& + K \int_{0}^{t} \Vert\delta X_s\Vert dt.
	\end{align}
	By Gronwall's inequality, we have
	\begin{align}
	\Vert\delta X_t \Vert \leq e^{KT} \int_{0}^{T} \Vert f(t,X^\theta_s,\phi_s) - f(t,X^\theta_s,\theta_s)\Vert ds.
	\label{eq:lem_1_dx_est}
	\end{align}
	To estimate $\delta P$, we use the same substitution as in Lemma~\ref{lem:H_Deriv} with $\tau=T-t$ and $\tilde{\cdot}_\tau=\cdot_{T-t}$. We get
	\[
	\delta \tilde{P}_\tau = \delta \tilde{P}_0 + \int_{0}^{\tau} \nabla_x H(t,\tilde{X}^\phi_s,\tilde{P}^\phi_s,\tilde{\phi}_s) - \nabla_x H(t,\tilde{X}^\theta_s,\tilde{P}^\theta_s,\tilde{\theta}_s) ds,
	\]
	and hence using Lemma~\ref{lem:H_Deriv} and assumptions (A1)-(A2), 
	\begin{align}
	\Vert\delta \tilde{P}_\tau \Vert \leq &\Vert\delta \tilde{P}_0\Vert + \int_{0}^{\tau} \Vert \nabla_x H(t,\tilde{X}^\phi_s,\tilde{P}^\phi_s,\tilde{\phi}_s) - \nabla_x H(t,\tilde{X}^\theta_s,\tilde{P}^\theta_s,\tilde{\theta}_s)\Vert ds \nonumber\\
	\leq & K \Vert\delta X_T \Vert + K K' \int_{0}^{T} \Vert\delta X_s\Vert ds + K\int_{0}^{\tau} \Vert\delta \tilde{P}_s \Vert ds \nonumber\\
	& + \int_{0}^{T} \Vert \nabla_x H(t,X^\theta_s,P^\theta_s,\phi_s) - \nabla_x H(t,X^\theta_s,P^\theta_s,\theta_s)\Vert ds \nonumber\\
	\leq & e^{KT}K(\Vert\delta X_T \Vert + K' \int_{0}^{T} \Vert\delta X_s\Vert ds) \nonumber\\
	& + e^{KT}\int_{0}^{T} \Vert \nabla_x H(t,X^\theta_s,P^\theta_s,\phi_s) - \nabla_x H(t,X^\theta_s,P^\theta_s,\theta_s)\Vert ds.
	\end{align}
	Using estimate~\eqref{eq:lem_1_dx_est}, we obtain
	\begin{align}
	\Vert\delta P_t \Vert \leq & e^{2KT} K(1+K'T) \int_{0}^{T} \Vert f(t,X^\theta_s,\phi_s) - f(t,X^\theta_s,\theta_s)\Vert ds \nonumber\\
	& + e^{KT} \int_{0}^{T} \Vert \nabla_x H(t,X^\theta_s,P^\theta_s,\phi_s) - \nabla_x H(t,X^\theta_s,P^\theta_s,\theta_s)\Vert ds.
	\label{eq:lem_1_dp_est}
	\end{align}
	Now, we substitute estimates~\eqref{eq:lem_1_dx_est} and~\eqref{eq:lem_1_dp_est} into~\eqref{eq:lem_1_cp6} and rename constants for simplicity. Note that by assumptions (A1)-(A2) and Lemma~\ref{lem:H_Deriv}, all the second derivative terms are bounded element-wise by some constant $K''$. Hence, we have $\vert\delta Z_t \cdot A \cdot \delta Z_t\vert\leq K''\Vert \delta Z \Vert^2$ for each $A$ being a second derivative matrix. Thus we obtain
	\begin{align}
	J(\phi) - J(\theta) \leq &-\int_{0}^{T} \Delta H_{\phi,\theta}(t)dt \nonumber\\
	& + \frac{1}{2}K'' \Vert\delta X_T\Vert^2 \nonumber\\
	& + K'' \int_{0}^{T}( \Vert\delta X_t\Vert^2 + \Vert\delta P_t\Vert^2) dt \nonumber\\
	& + \frac{1}{2} \int_{0}^{T} \Vert \delta X_t \Vert \Vert f(t,X^\theta_t,\phi_t) - f(t,X^\theta_t,\theta_t) \Vert dt \nonumber\\
	& + \frac{1}{2} \int_{0}^{T} \Vert \delta P_t \Vert \Vert \nabla_x H(t,X^\theta_t, P^\theta_t, \phi_t) - \nabla_x H(t,X^\theta_t, P^\theta_t, \theta_t) \Vert dt	\nonumber\\
	\leq &-\int_{0}^{T} \Delta H_{\phi,\theta}(t)dt \nonumber\\
	&+ C \left(\int_{0}^{T} \Vert f(t,X^\theta_t,\phi_t) - f(t,X^\theta_t,\theta_t) \Vert dt\right)^2 \nonumber\\
	&+ C \left(\int_{0}^{T} \Vert \nabla_x H(t,X^\theta_t, P^\theta_t, \phi_t) - \nabla_x H(t,X^\theta_t, P^\theta_t, \theta_t) \Vert^2 dt\right)^2\\
	\leq &-\int_{0}^{T} \Delta H_{\phi,\theta}(t)dt \nonumber\\
	&+ C \int_{0}^{T} \Vert f(t,X^\theta_t,\phi_t) - f(t,X^\theta_t,\theta_t) \Vert^2 dt \nonumber\\
	&+ C \int_{0}^{T} \Vert \nabla_x H(t,X^\theta_t, P^\theta_t, \phi_t) - \nabla_x H(t,X^\theta_t, P^\theta_t, \theta_t) \Vert^2 dt.\nonumber
	\end{align}
\end{proof}

\begin{remark}
	For applications, the global Lipschitz condition (A2) w.r.t. $x$ on $f$ may be restrictive. Note that this can be replaced by a local Lipschitz condition if we can show that $X_t$, $t\in[0,T]$ is bounded for all $\theta\in\mathcal{U}$. This is true if the parameter space $\Theta$ is bounded, which we can safely assume in practice, as long as a suitable regularization is used that prevents the parameters from getting arbitrarily large. Alternatively, a projection step can be used to restrict the parameters to a bounded set. In either cases, this should not negatively affect the performance of the model. 
\end{remark}

\bibliography{ref}

\end{document}